\documentclass{article}

\usepackage{amsmath}
\usepackage{amsfonts}
\usepackage{amsthm}
\usepackage{subfigure}

\usepackage{tablefootnote}

\usepackage{mathtools}

\usepackage{colortbl}
\definecolor{LightCyan}{rgb}{0.8, 0.9, 1}

\usepackage{amssymb}
\usepackage{pifont}

\newcommand{\rE}{\mathbb E}

\newcommand{\KL}{\mathsf{KL}}

\newcommand{\la}{\left\langle}
\newcommand{\ra}{\right\rangle}

\newcommand{\E}{\rE}
\newcommand{\seq}[1]{\overline{[#1]}}

\def\final{}
\ifdefined\final
\usepackage[disable]{todonotes}
\else
\usepackage[textsize=tiny]{todonotes}
\fi
\allowdisplaybreaks

\newcommand{\piref}{\pi^{\mathsf{ref}}}
\newcommand{\kl}[2]{\ensuremath{{\mathsf{KL}}\left(#1\|#2\right)}}
\newcommand{\tv}[2]{\ensuremath{{\mathsf{TV}}\left(#1\|#2\right)}}
\newcommand{\confbandittb}{\cE_{\mathrm{tab}}}
\newcommand{\confbanditbin}{\cE_{\mathrm{bin}}}
\newcommand{\confbanditfa}{\cE_{\mathrm{gfa}}}

\newcommand{\fls}{{\bar{g}}} 
\newcommand{\fps}{{\hat{g}}} 

\newcommand{\fgt}{{g^*}} 

\newcommand{\fcl}{{\cG}} 

\def \algcb {\text{FKL-PCB}}

\newcommand{\obj}{J}

\newcommand{\pistar}{\pi^*}
\newcommand{\subopt}{\mathrm{SubOpt}}

\newcommand{\objfkl}{J_{\mathrm{FKL}}}
\newcommand{\pistarfkl}{{\pi^*_{\mathrm{FKL}}}}
\newcommand{\suboptfkl}{\mathrm{SubOpt}_{\mathrm{FKL}}}

\DeclareMathOperator{\ri}{ri}
\newcommand{\supp}{\mathrm{supp}}

\newcommand{\breg}[3]{\ensuremath{{\mathsf{B}}_{#1}\left(#2, #3\right)}}

\usepackage{colortbl}
\definecolor{LightCyan}{rgb}{0.8, 0.9, 1}

\usepackage{soul}

\usepackage{smile}

\usepackage[main, final]{neurips_2026}

\usepackage[utf8]{inputenc} 
\usepackage[T1]{fontenc}    
\usepackage{url}            
\usepackage{booktabs}       
\usepackage{amsfonts}       
\usepackage{nicefrac}       
\usepackage{microtype}      
\usepackage{xcolor}         

\usepackage[compact]{titlesec}
\titlespacing*{\section}
{0pt}{1.1ex}{1.1ex}
\titlespacing*{\subsection}
{0pt}{1.1ex}{1.1ex}
\titlespacing*{\subsubsection}
{0pt}{1.1ex}{1.1ex}

\title{Fast Rates for Offline Contextual Bandits with Forward-KL Regularization under Single-Policy Concentrability}

\author{Qingyue Zhao$^{1}$\thanks{Equal Contribution} ,
Kaixuan Ji$^{1}$\textcolor{red}{$^*$}, 
Heyang Zhao$^{1}$,
Quanquan Gu$^{1}$\\
$^1$University of California, Los Angeles\\
\texttt{\{zhaoqy24,kaixuanji,hyzhao,qgu\}@cs.ucla.edu}
}

\begin{document}

\maketitle

\begin{abstract}
  \emph{Kullback-Leibler} (KL) regularization is ubiquitous in reinforcement learning algorithms in the form of \emph{reverse} or \emph{forward} KL. Recent studies have demonstrated $\epsilon^{-1}$-type fast rates for decision making under reverse KL regularization, in contrast to the standard $\epsilon^{-2}$-type sample complexity. However, for forward-KL-regularized objectives, existing statistical analyses are either not applicable or result in $\tilde{O}(\epsilon^{-2})$ slow rates. We take the first step towards addressing this problem via a streamlined analysis of forward-KL-regularized offline CBs. We give the first $\tilde{O}(\epsilon^{-1})$ upper bounds in tabular and general function approximation settings, both under notions of \emph{single-policy concentrability}. In particular, our convex-analytical pipeline unifies these settings by exploiting the pessimism principle in a novel way and completely bypasses the proof routines in previous works based on the mean value theorem, which might be of independent interest. Moreover, we provide rate-optimal lower bounds, manifesting the tightness of our upper bounds in terms of statistical rates. Our lower bounds also demonstrate that the forward-KL-regularized sample complexity recovers the unregularized slow rate in the low-regularization regime, similarly to the reverse-KL regularization.
\end{abstract}

\section{Introduction}

KL-regularized objectives in the form of $J \coloneqq \EE [r] - \eta^{-1} \KL$ have recently become a pivotal prototype for algorithm design in decision making such as robotics~\citep{levine2013guided,schulman2015trust,haarnoja2018soft}, and foundation model finetuning~\citep{ouyang2022training,rafailov2023direct,ji2023language,wang2023beyond,guo2025deepseek,shan2025forward}. Here, $r$ stands for variants of reward or advantage functions, $\eta^{-1}$ is the regularization intensity; and $\KL$ is typically either $\kl{\pi}{\piref}$ or $\kl{\piref}{\pi}$, where $\piref$ is the reference policy. The reverse KL regularizer $\kl{\pi}{\piref}$ has been arguably more popular in algorithmic formulations, which has rich connections with Gibbs distributions and entropy regularization in machine learning~\citep{williams1992simple,mcallester1999pac,ziebart2008maximum,zhang2023ltbook}. But there are also emerging lines of works that formulate their RL fine-tuning algorithms via the forward KL regularizer $\kl{\piref}{\pi}$~\citep{ji2023language,wang2023beyond,shan2025forward}, or \emph{effectively} employ this forward counterpart at the implementation level~\citep{guo2025deepseek,tang2025few,shah2025comedy}.

\newcolumntype{g}{>{\columncolor{LightCyan}}c}
\begin{table*}[h]
\centering
\caption{Comparison of regret or sample complexity upper and lower bounds for forward KL-regularized CBs. In this table, $\epsilon$ is the target suboptimality gap, and $\eta$ is the inverse regularization intensity. Under function approximation, $\fcl$ is the function class, whose covering number is $N_{\fcl}$, and $D_{\pi^*}^2$ is the $D^2$-type single policy concentrability. In the tabular setting, $\cS$ is the context set and $\cA$ is the action set, and $|\cS|$ and $|\cA|$ denotes the corresponding cardinality. Furthermore, $C_{\varepsilon_{\text{FKL}}}$ is the density-ratio-based local forward KL-ball coverage coefficient, and $C^{\pi^*}$ is the density-ratio-based single-policy concentrability. $\tilde{O}(\cdot)$ hides logarithmic factors except $\log (N_{\fcl})$.
}
\renewcommand{\arraystretch}{1}
\resizebox{0.99\columnwidth}{!}{
{
\begin{tabular}{cgggg}
\toprule
\rowcolor{white}Type &Algorithm  & Setting& Sample Complexity \\ 
\midrule
\rowcolor{white}
& Forward KL-Regularized RLHF \\\rowcolor{white}
& \small\citep{aminian2025kl} & \multirow{-2}{*}{Function Approximation}&    \multirow{-2}{*}{$\tilde{O}\big(\log {N}_{\fcl} C_{\varepsilon_{\text{FKL}}}/\color{red}\epsilon^2\color{black}\big)$}    \\
& \algcb & & \\
&\small(This Work) & \multirow{-2}{*}{Function Approximation} &     \multirow{-2}{*}{$\tilde{O}\big(\eta C^{\pi^*} D_{\pi^*}^2 \log N_{\fcl}/\color{red}\epsilon\color{black}\big)$}      \\
\multirow{-4}{*}{Upper Bound}  & \algcb & & \\
&\small(This Work) & \multirow{-2}{*}{Tabular} &     \multirow{-2}{*}{$\tilde{O}\big(\eta (C^{\pi^*})^2 |\cS||\cA| /\color{red}\epsilon\color{black}\big)$}      \\
\midrule
\multirow{-1}{*}{Lower Bound}  & This Work & Tabular & $ \Omega(\eta |\cS||\cA| / \color{red}\epsilon\color{black})$   \rule{0pt}{3.5ex}\\[4pt] 
\bottomrule
\end{tabular}
}}
\label{tab:kl-regularized}
\end{table*}

There has been a line of efforts motivated by the practical relevance of KL-regularized objectives towards understanding the \emph{data efficiency} of learning with respect to \emph{KL-regularized performance metrics}, which dates back to at least \citet{tiapkin2023fast,xie2024exploratory,xiong2024iterative}. In the case of reverse KL, sharpest rates in previous works have been $\polylog(T)$ regret~\citep{zhao2025logarithmic,ji2026near,nayak2025achieving} and $\tilde{\Theta}(\epsilon^{-1})$ sample complexity~\citep{tiapkin2023fast,zhao2025sharp,zhao2026towards,foster2025good,zhang2026beyond} in offline, online, and hybrid settings for contextual bandits (CBs) or Markov decision processes (MDPs). The analysis of the forward KL counterparts is significantly less explored.
\citet{zhao2026towards} derives $\tilde{\Theta}(\epsilon^{-1})$ bounds for general $f$-divergence-regularized suboptimality for offline CBs, but they need $f$ to be strongly convex, and is thus inapplicable to forward KL because forward KL is $f$-divergence with $f(x) = -\log x$, which is only strictly convex. \citet{lee2026regularized} derives a logarithmic regularized regret upper bound with general regularizers in a self-play setting, but the bound directly scales inversely with the minimal eigenvalue of the expected feature covariance matrix, and their regularizer needs to be strongly convex, which may not be the case for forward KL unless the sample complexity is allowed to inversely scale with the minimal mass $\min_{s,a}\piref(a|s)$ of the reference policy.\footnote{See \Cref{lem:fkl:sc} for the concrete exposition.} \citet{aminian2025kl} gives the first $\tilde{O}(\epsilon^{-2})$ upper bound for offline CBs with forward-KL regularization, where no lower bounds are known by far.
Therefore, the following problem remains open \textbf{even for learning from pure $\iid$ data}:
\begin{center}
    \emph{What is the statistical rate of {decision making} with respect to} forward-KL-regularized objectives?
\end{center}
Any meaningful progress regarding this problem in the offline setting is likely to confront the distributional shift between the $\iid$ data and the optimal policy. On one hand, \emph{single-policy concentrability}, where the behavioral policy $\piref$ only covers the optimal policy, is shown to be a sufficient notion that permits the near-optimal $\tilde{\Theta}(1/\epsilon)$ sample complexity in many \textbf{reverse}-KL-regularized settings where the contexts and actions are acquired offline~\citep{zhao2026towards,foster2023foundations}. On the other hand, classic wisdom implies that the forward-KL regularizer and its reverse counterpart adapt $\pi$ to the modes of $\piref$ in distinct ways~\citep{bishop2006pattern,ji2023language,gx2025kl}: in short, (in certain parametrized cases) reverse KL is believed to be mode-seeking, which heavily penalizes $\pi > 0$ where $\piref \approx 0$; while forward KL appears to be mass-covering, which mainly penalizes $\pi \to 0$ where $\piref > 0$. The following problem then arises \textbf{on top of the statistical rate problem} above:
\begin{center}
    \emph{Is} single-policy concentrability \emph{also sufficient for rate-optimal offline learning with forward-KL regularization}?
\end{center}
We solve these problems for offline CBs concretely by the $\tilde{\Theta}(\epsilon^{-1})$ sample complexity characterization\footnote{We omit the dependencies on quantities other than $\epsilon$ here for presentation clarity.}, which is achieved under single-policy concentrability assumptions for the first time. At the core of our achievability results is a novel suboptimality decomposition for forward-KL-regularized objectives, overcoming the limitation of the traditional synergy of Taylor expansion and the mean value theorem \citep{zhao2025sharp,zhao2026towards,zhao2025logarithmic,nayak2025achieving,ji2026near} for analyzing the divergence-regularized sample complexity of bonus-based RL algorithms, which leads to a streamlined analysis of \emph{\textbf{F}orward \textbf{KL}-Regularized \textbf{P}essimistic \textbf{C}ontextual \textbf{B}andits} (\algcb{}) algorithms without any standalone dependency on $\min_{s,a}\piref(a|s)$. Our hardness results further demonstrate the fundamental rate limits in both the high- and low-regularization regimes.

\noindent\textbf{Notation.}
The calligraphic $\cS$ and $\cA$ denote finite sets throughout this paper. $\Delta(\cA)$ is the set of probability distributions on $\cA$ and $\Delta(\cA|\cS)$ is the family of probability kernels from $\cS$ to $\cA$.
For $P, Q \in \Delta(\cA)$ with $P \ll Q$, the KL divergence from $P$ to $Q$ is denoted by $\kl{P}{Q} \coloneqq \int \log({\ud P}/{\ud Q})\ud P$, and supp$(P)$ denotes the support set of $P$. For any subset $E$ of a Euclidean space, $\ri(E)$ denotes its \emph{relative interior}. For any extended real-valued convex function $f$, its \emph{domain} is $\mathrm{dom}(f) \coloneqq \{\pi: f(\pi) < +\infty\}$. Any closed convex function $\ell: \RR^d \to \RR \cup {+\infty}$ induces a Bregman divergence $\breg{\ell}{\xb}{\yb} \coloneqq \ell(\xb) - \ell(\yb) - \la \nabla \ell(\yb), \xb - \yb \ra$ for $\xb, \yb \in \RR^d$. Given  $f, g: \cA \to \RR \cup \{+\infty\}$, $\dotp{f}{g}$ denotes the summation $\sum_{a\in \cA} f(a)g(a)$. We use standard asymptotic notations including $O(\cdot)$, $\Omega(\cdot)$, and $\Theta(\cdot)$; for which $\tilde O(\cdot), \tilde\Omega(\cdot), \tilde\Theta(\cdot)$ further hide $\polylog$ factors.

\subsection{Key Related Works}

We survey recent sharp sample complexity analyses of decision making with respect to regularized objectives and defer other related works on pessimism and forward KL regularization to \Cref{app:add-related}.

\noindent\textbf{Fast Rates for Regularized Decision Making.}
The practical relevance of regularized objectives has motivated a broad line of sharper analyses of the statistical efficiency of decision-making under such formulations. The pioneering work of~\citet{tiapkin2023fast} first established an $\epsilon^{-1}$-type fast sample complexity in the pure-exploration setting, followed by a series of works~\citep{zhao2025logarithmic,ji2026near} that focused on regret minimization. A parallel line of research focuses on the offline or hybrid setting, starting from~\citet{zhao2025sharp} that obtained a $\tilde{\Theta}(\epsilon^{-1})$ rate under concentrability of all policies. This requirement was relaxed in subsequent works~\citep{foster2025good,kim2026coverage,zhao2026towards}, in the context of adaptive reward-query, batched online learning, and pure offline learning, respectively. Similar fast convergence has also been shown in online~\citep{nayak2025achieving} and offline \citep{zhang2026beyond} multi-agent RL, and self-play CBs with generalized bilinear function approximation \citep{lee2026regularized}. Despite these advances, the literature has largely focused on reverse KL or regularizers with certain strong-convexity-type properties, leaving a wide range of regularizers, such as forward KL, comparatively underexplored. Notably, \citet{aminian2025kl} gave the first algorithm analysis for offline contextual dueling bandits with multiple reference models under forward-KL regularization; which, however, yields an $\tilde{O}(\epsilon^{-2})$ slow rate.

\section{Preliminaries} \label{sec:prelim}

Consider offline CBs with context space $\cS$, action space $\cA$, mean reward function $r: \cS \times \cA \to [0, 1]$, context distribution $\rho \in \Delta(\cS)$, and reference policy $\piref \in \Delta(\cA |\cS)$. The agent only learns from an $\iid$ dataset $\cD = \{s_i, a_i, r_i\}_{i=1}^n$, where $s_i \sim \rho$, $a_i \sim \piref(\cdot|s_i)$, and $r_i = r(s_i, a_i) + \varepsilon_i$ for each $i$. Here, each $\varepsilon_i$ is conditionally mean-zero and $1$-subgaussian given $(s_i, a_i)$, and $\piref$ \emph{serves as the behavioral policy}. In particular, the class of interest satisfying all constraints above is denoted by $\mathrm{CB}(\cS, \cA, r, \rho, \piref)$.\footnote{In single-context or context-wise derivations with fixed $s \in \cS$, we write $r(a)$ as shorthand for $r(s, a)$.}
Given the inverse regularization intensity $\eta > 0$, let the forward-KL-regularized objective be
\begin{align}
    \objfkl(\pi) \coloneqq \objfkl(\pi; r) \coloneqq \dotp{r}{\pi} - \eta^{-1}\kl{\piref}{\pi}, \label{eq:fkl:offline-obj}
\end{align}
where we overload $\dotp{r}{\pi} \coloneqq \EE_{s \sim \rho} \dotp{r(s,\cdot)}{\pi(\cdot|s)}$ and $\kl{\piref}{\pi} \coloneqq \EE_{s\sim\rho} \kl{\piref(\cdot|s)}{\pi(\cdot|s)}$. The learning objective is $\max_{\pi \in \Delta(\cA|\cS)} \objfkl(\pi)$. We assume $\piref$ has full support to ensure learnability over $\Delta(\cA|\cS)$ and simplify the presentation.
\begin{assumption}\label{assump:full-supp}
    $\min_{s,a \in \cS \times \cA} \piref(a|s)  > 0$.
\end{assumption}
If $\supp(\piref(\cdot|s)) \subsetneq \cA$, $\cD$ never contains actions outside $\cA_s^{\mathrm{eff}} \coloneqq \supp(\piref(\cdot|s))$. Unlike reverse KL (which evaluates to $+\infty$ outside $\cA_s^{\mathrm{eff}}$), $\kl{\piref(\cdot|s)}{\pi(\cdot|s)}$ remains finite when $\pi(\cdot|s)$ places mass outside $\cA_s^{\mathrm{eff}}$, so unrestricted optimization over $\Delta(\cA|\cS)$ can be unidentifiable from offline data (specifically whenever $\eta > 1$) without cross-action structural assumptions; whereas optimizing over the support-restricted class $\Pi_{\mathrm{eff}} \coloneqq \{\pi \in \Delta(\cA|\cS): \supp(\pi(\cdot|s)) \subseteq \cA_s^{\mathrm{eff}}, \forall s \in \cS\}$ reduces the action set to $\cA_s^{\mathrm{eff}}$, where $\piref$ has full support by definition (see \Cref{app:subsec:full-supp}). Thus, the central statistical question is whether the sample complexity blows up as $\min_{s, a \in \cA_s^{\mathrm{eff}}} \piref(a|s) \to 0$.
\Cref{assump:full-supp} has also been used in \citet[Lemma~F.4]{huang2024correcting} to derive a closed-form solution of a regularized optimal policy, and in the analyses of \citet{wang2023beyond,aminian2025kl} implicitly in a similar way. It does \emph{not} impose a quantitative gap between $\min_{s,a}\piref(a|s)$ and $0$, which is also not needed in our analysis.
The following closed-form of the forward-KL-regularized optimal policy follows from the classic KKT theory (See, e.g., \citet[Appendix~A]{beck2017first} and \citealp{wang2023beyond,aminian2025kl}). We provide its proof in \Cref{app:other-missing} for completeness.

\begin{lemma}\label{lem:fkl:opt}
    Under \Cref{assump:full-supp}, $\forall s \in \cS$, there exists a unique $\lambda_s > \max_{a \in \cA} r(s, a)$ such that
\begin{align}
    \pistarfkl(\cdot |s) \coloneqq \frac{\eta^{-1}\piref(\cdot |s)}{\lambda_s - r(s,\cdot)} \label{eq:fkl:sol}
\end{align}
is the unique solution to $\max_{\pi\in\Delta(\cA|\cS)}\objfkl(\pi)$.
\end{lemma}
We then define the sub-optimality gap \emph{regularized by forward KL} as
\begin{align}
   \suboptfkl(\cdot) \coloneqq \suboptfkl(\cdot; \cS, \cA, r, \rho, \piref) = \objfkl(\pistarfkl) - \objfkl(\cdot),
\end{align}
under which, the agent aims to find the \emph{$\epsilon$-optimal policy} $\hat\pi$ with $\suboptfkl(\hat\pi) \leq \epsilon$. For simplicity, we use $\pi^*$ interchangeably with $\pistarfkl$ to avoid notation clutter unless otherwise specified.


\noindent\textbf{Concentrability.} In offline learning and decision making, how the dataset $\cD$ collected by $\piref$ covers a target distribution or a collection of potential target distributions largely shapes the learnability of the problem. This coverage is typically characterized by \emph{the} concentrability in offline RL, which quantifies the ability of the behavior policy to generate diverse actions.
We first introduce the density-ratio-based concentrability~\citep{chen2019information,xie2021bellman,rashidinejad2021bridging,li2024settling} as follows, which characterizes how each pair of $(s,a) \in \cS \times \cA$ is covered by the behavior policy. 

\begin{definition}[\emph{Density-ratio-based} concentrability]\label{def:density-ratio-concentrability}
Given some policy class $\Pi$ and some reference policy $\piref$, the density-ratio-based all-policy concentrability $C^{\Pi}$ is defined by $C^{\Pi} \coloneqq \sup_{\pi \in \Pi, s \in \cS, a \in \cA}{\pi(a|s)} / {\piref(a|s)}$, whose single-policy counterpart under the optimal policy $\pistar$ is $C^{\pistar} \coloneqq \sup_{s\in \cS, a \in \cA} {\pistar(a|s)} / {\piref(a|s)}$.
\end{definition}

When the optimal policy $\pi^*$ is contained within the policy class $\Pi$, the relationship $C^{\pi^*} \le C^{\Pi}$ necessarily holds. Consequently, an analytical dependency on the single-policy concentrability coefficient $C^{\pi^*}$ is strictly weaker, and thus more general, than a dependency on the all-policy coefficient $C^{\Pi}$. In this work, we demonstrate that this more relaxed single-policy concentrability is sufficient for learning forward-KL-regularized bandits. This result stands in contrast to the more restrictive requirement of all-policy concentrability typically invoked in previous literature~\citep{zhao2025sharp, wu2025greedy, aminian2025kl}.

\subsection{Function Approximation}

While the tabular setting serves as the minimal decision making formalization without additional structures, algorithms designed directly for this setting become computationally intensive or even intractable in most real-world applications due to the curse of dimensionality. To address this issue, it is generally assumed that the learner has access to some class of functions that incorporate the ground truth reward~\citep{jin2021pessimism,xiong2022nearly,di2023pessimistic}. In this paper, we also consider the scenarios that the reward lies in some known function class $\fcl$, as formalized by the following realizable assumption.

\begin{assumption}\label{assume:general-function-approx}
    For this known function class $\fcl \subseteq (\cS \times \cA \to [0,1])$, $\exists\fgt \in \fcl$ with $\fgt = r$. 
\end{assumption}
To characterize the complexity measure of the reward function class, we adopt the standard notion of covering number \citep[Definition~5.1]{wainwright2019high}.

\begin{definition}[{$\epsilon$-net and covering number}]
Given a function class $\cG \subset (\cS \times \cA \to \RR)$, a finite set $\cG(\epsilon) \subset \cG$ is an $\epsilon$-net
of $\cG$ w.r.t. $\|\cdot\|_\infty$, if for any $g \in \cG$, there exists $g' \in \cG(\epsilon)$ such that $\| g - g'\|_\infty \leq \epsilon$. The $\epsilon$-covering number is the smallest cardinality $N_{\cG}(\epsilon)$ of such $\cG(\epsilon)$.
\end{definition}

\begin{assumption}\label{assume:poly-covering}
For any $\epsilon_c > 0$, the $\epsilon_c$-covering number $N_{\fcl}(\epsilon_c)$ of $\fcl$ is $\poly(\epsilon_c^{-1})$.
\end{assumption}
\Cref{assume:poly-covering} allowing $\log N_{\fcl}(\epsilon)$ to be roughly negligible is arguably mild. For example, when $\fcl$ is the class of linear functions of dimension $d$ and radius $R$, the covering number is $N_{\fcl}(\epsilon) = O((1+R\epsilon^{-1})^d)$ \citep[Lemma~D.6]{jin2020provably}, which satisfies \Cref{assume:poly-covering}.

In the presence of function approximation, the density-ratio-based concentrability in Definition~\ref{def:density-ratio-concentrability} does not precisely capture the quality of the offline dataset due to the absence of the function approximation structure in its definition. For example,  a particular context-action pair $(s,a) \in \cS \times \cA$ might be frequently visited by $\pi^*$ but seldom by $\piref$, leading to very large $C^{\pi^*}$. However, if there is some $(s',a') \in \cS \times \cA$ frequently appearing in the dataset $\cD$ such that $g(s,a) \approx g(s',a')$ for all $g \in \cG$, then the reward of $(s,a)$ is still well estimated, since feature representation associated with $(s,a)$ is frequently visited by visiting $(s',a')$. To capture this dependency, we introduce the following $D^2$-type concentrability~\citep{gentile2022achieving,agarwal2023vo,zhao2025sharp} in function approximation, which characterize the coverage of the behavior policy on the feature space.

 \begin{definition}\label{def:bandit:D-sq}
Given function class $\fcl \subset ( \cS \times \cA \to \mathbb{R})$ and policy $\pi$, their $D^2$-divergence is
\begin{align*}
  D^2_{\fcl}((s,a);\pi) \coloneqq  \sup_{ g,h \in \fcl} \frac{\big( g(s, a) - h(s, a)\big)^2}{\E_{(s',a') \sim \rho \times \pi}[( g(s', a') - h(s', a'))^2]}, \forall (s,a) \in \cS \times \cA.
\end{align*}
\end{definition}

We are now ready to define the $D^2$-type single policy concentrability. 
\begin{assumption}[Single-policy concentrability]\label{assume:single-coverage-bandit}
    $D_{\pi^*}^2 \coloneqq \E_{(s,a) \sim \rho \times \pi^*} D^2_{\fcl}((s,a); \piref) <\infty$.
\end{assumption}
\Cref{assume:single-coverage-bandit} indicates that the errors on the context-action distributions   $\rho \times \pi^*$ can be bounded by the error on the samples from  $\rho \times \piref$ up to some constant. Similar to the density-ratio-based concentrabilities, here it is also true that the single-policy concentrability assumption is strictly weaker than the all-policy concentrability assumption (e.g., \citealt[Assumption 2.7]{zhao2026towards}). 
Remarkably, the two quantities characterizing single-policy concentrability $C^{\pi^*}$ and $D^2_{\pi^*}$ cannot be bounded by each other up to constant factors in general. We provide a detailed discussion of the relation between $C^{\pi^*}$ and $D^2_{\pi^*}$ in Appendix~\ref{app:discussion-coverage}.

\color{black}

\section{Algorithms and Sample Complexity Upper Bounds}\label{sec:ub}

In this section, we present the algorithm, $\algcb$, for learning forward-KL-regularized CBs. For ease of presentation, we present its adaptations for both tabular setting and reward function approximation respectively, in Section~\ref{sec:ub-tab} and Section~\ref{sec:ub-gfa}.

\begin{algorithm*}[t]
\caption{$\algcb$ for Tabular Setting}
\begin{algorithmic}[1]\label{algorithm:bandit-pess-tb}
    \REQUIRE regularization $\eta$, reference policy $\piref$, offline dataset $\cD$

    \STATE Set $N(s,a) = \sum_{i=1}^n \ind \{(s_i, a_i) = (s,a)\}$ for all $a \in \cA$, $s \in \cS$

    \FOR{$s \in \cS$, $a \in \cA$}
    \IF{$N(s,a)=0$}
    \STATE Set the empirical reward $\fls(s,a) \leftarrow 0$, penalty $b(s,a) \leftarrow 1$
    \ELSE
    \STATE Compute the empirical reward $\fls(s,a) \leftarrow  \sum_{i=1}^n r_i \ind \{(s_i, a_i) = (s,a)\} / {N(s,a)}$
    \STATE Compute the penalty to be 
    \begin{align}
        b(s,a) \leftarrow \sqrt{ {4\log(2|\cS||\cA|/\delta)} \big/{N(s,a)}},\label{eq:bandit-pess-tabular}
    \end{align}
    \ENDIF
    \STATE Set $\fps(s,a) \leftarrow \max\{\fls(s,a) - b(s,a), 0\}$
    \ENDFOR
    \ENSURE $\hat \pi = \argmax_{\pi \in \Delta(\cA|\cS)}\dotp{\fps}{\pi} - \eta^{-1}\kl{\piref}{\pi}$
\end{algorithmic}
\end{algorithm*}

\subsection{Algorithm for Tabular Setting}\label{sec:ub-tab}

We first tackle the most vanilla tabular setting, where the reward function can be any map in $\cS \times \cA \to [0,1]$, as summarized in Algorithm~\ref{algorithm:bandit-pess-tb}. At a high level, the algorithm follows previous pessimism-based algorithms~\citep{rashidinejad2021bridging} for regularized objectives~\citep{zhao2026towards}. In particular, Algorithm~\ref{algorithm:bandit-pess-tb} first employs a least-squares estimator to estimate the reward function. Then following~\citet{zhao2026towards}, we subtract a penalty term to prevent over-estimation, leading to a pessimistic reward estimate. In the tabular case, we adopt the standard tabular confidence bound~\citep{xie2021policy}, leading to the penalty given by~\eqref{eq:bandit-pess-tabular}. The following results show that $\fps$ is indeed a pessimistic estimate with high probability and the number of samples $N(s,a)$ does not deviate too much from its expectation $\rho(s)\piref(a|s)$.

\begin{lemma}\label{lem:pessimistic-tabular}
Given $\delta \in (0, 1)$, let $\fls(s,a)$ and $b(s,a)$ be as defined in \Cref{algorithm:bandit-pess-tb}, and let $\confbandittb(\delta)$ denote the event that the reward estimation error is uniformly controlled:
\begin{align}\label{eq:confbandit-tabular}
    \confbandittb(\delta) \coloneqq \Big\{  \big| \fls(s,a) - \fgt(s,a)\big| \leq b(s,a) \text{, for all } (s,a) \in \cS \times \cA \Big\}.
\end{align}
We further use $\confbanditbin(\delta)$ to denote the event under which $N(s,a)$ does not deviate too much from the expectation, that is
\begin{align*}
    \confbanditbin(\delta) \coloneqq \bigg\{ \frac{1}{N(s,a) \vee 1} \leq \frac{8\log(2|\cS||\cA|/\delta)}{n\rho(s)\piref(a|s)} \text{ for all } (s,a) \in \cS \times \cA \bigg\}.
\end{align*}
Then the event $\confbandittb(\delta) \cap \confbanditbin(\delta)$ holds with probability at least $1-\delta$.
\end{lemma}

Finally, we construct the output policy to be the regularized maximizer under the clipped pessimistic reward $\fps$. The following theorem provides the sample complexity guarantee of Algorithm~\ref{algorithm:bandit-pess-tb}; the proofs of \Cref{lem:pessimistic-tabular,thm:tabular} are deferred to \Cref{app:ub-missing-proof}.

\begin{theorem}\label{thm:tabular}
For sufficiently small $\epsilon \in (0, 1)$, with probability at least $1 - \delta$, $\tilde{O}\big({\eta (C^{\pi^*})^2 |\cS| |\cA| \epsilon^{-1}}  \big)$ samples
suffice to guarantee the learned policy $\hat \pi$ of \Cref{algorithm:bandit-pess-tb} to be $\epsilon$-optimal.
\end{theorem}

\subsection{Algorithm for Function Approximation}\label{sec:ub-gfa}

\begin{algorithm*}[t]
	\caption{$\algcb$ for Function Approximation}\label{algorithm:bandit-pess-gfa}
	\begin{algorithmic}[1]
    \REQUIRE regularization $\eta$, reference policy $\piref$, offline dataset $\cD$, function class $\fcl$

    \STATE Reward estimation via least squares: $\fls \in \argmin_{g \in \fcl} \sum_{(s_i,a_i, r_i) \in \cD} \big(g(s_i, a_i) - r_i\big)^2$ \label{line:bandit-lsq}
        
    \STATE Compute the pessimism penalty with $\beta = \tilde{O}\Big(\sqrt{\log\big(2N_{\fcl}(\epsilon)/\delta\big)/n + \epsilon}\Big)$
    \begin{align}
        \Gamma_n(s,a)= \beta  D_{\fcl}\big((s,a),\piref\big), \forall (s,a) \in \cS \times \cA \label{eq:bandit-pess-function-approximation}
    \end{align}
    \STATE Compute the pessimistic reward function $\fps(s,a) \leftarrow \max\{\fls(s,a) - \Gamma_n(s,a), 0\}, \forall (s,a) \in \cS \times \cA$
    \ENSURE $\hat \pi = \argmax_{\pi \in \Delta(\cA|\cS)}\dotp{\fps}{\pi} - \eta^{-1}\kl{\piref}{\pi}$
\end{algorithmic}
\end{algorithm*}

While the algorithm follows the same backbone as \Cref{algorithm:bandit-pess-tb}, access to the function class $\fcl$ allows us to perform least-squares estimation within $\fcl$. After the estimation, we construct the pessimism penalty as in~\eqref{eq:bandit-pess-function-approximation}, following~\citet{di2023pessimistic,zhao2026towards}. For simplicity, \Cref{algorithm:bandit-pess-gfa} uses the population disagreement width $D_{\fcl}((s,a); \piref)$ and assumes oracle access to least-squares regression and bonus evaluation over $\fcl$ (while $\hat{\pi}$ admits the context-wise closed form in \Cref{lem:fkl:opt}). When $\rho$ is unknown, replacing $D_{\fcl}$ with a sample-split regularized empirical disagreement width preserves the guarantee of \Cref{thm:function-approximation} up to universal constants and logarithmic factors; see \Cref{app:subsec:empirical-bonus} for the explicit construction, oracle model, and tractable linear/tabular instantiations. Since $\fgt \in [0, 1]$, this penalty ensures $0 \leq \fps = \max\{\fls - \Gamma_n, 0\} \leq \fgt$ with high probability. Formally, we define the event $\confbanditfa(\delta)$ given $\delta>0$ as
\begin{align}\label{eq:confbandit-function-approximation}
    \confbanditfa(\delta) \coloneqq \Big\{ \sup\nolimits_{(s,a) \in \cS \times \cA} \Big[ \big| \fls - \fgt\big| - \Gamma_n \Big] (s,a) \leq 0 \Big\}.
\end{align}
The following lemma indicates that $\confbanditfa(\delta)$ holds with probability at least $1- \delta$.

\begin{lemma}[Lemma 2.9, \citealt{zhao2026towards}]\label{lem:pessimistic-function-approximation}
$\forall \delta \in (0,1)$, $\confbanditfa(\delta)$ holds with probability at least $1-\delta$.
\end{lemma}

Finally, we construct the output policy as the maximizer of the pessimistic reward $\fps$ under the regularized objective. The following theorem provides the sample complexity guarantee of Algorithm~\ref{algorithm:bandit-pess-gfa} (proved in \Cref{app:ub-missing-proof}).

\begin{theorem}\label{thm:function-approximation}
For sufficiently small $\epsilon \in (0, 1)$, $\tilde{O}\big({\eta C^{\pi^*} D_{\pi^*}^2 \epsilon^{-1}\log N_{\fcl}(\epsilon)} \big)$ samples
suffice to guarantee the learned policy $\hat \pi$ of \Cref{algorithm:bandit-pess-gfa} to be $\epsilon$-optimal with probability at least $1 - \delta$.
\end{theorem}

\begin{remark}
    Previously, \citet{aminian2025kl} obtained a $\tilde{O}\big(\log {N}_{\fcl} C_{\varepsilon_{\text{FKL}}}/\epsilon^2\big)$ sample complexity guarantee for offline learning with forward-KL-regularized objective. In contrast, both Theorem~\ref{thm:tabular} and Theorem~\ref{thm:function-approximation} establish an $\epsilon^{-1}$-fast rate, improving upon the previous $\epsilon^{-2}$ dependency. Moreover, the sample complexity in both Theorem~\ref{thm:tabular} and Theorem~\ref{thm:function-approximation} only relies on the notion of single policy concentrability, thereby relaxing the stronger all-policy concentrability assumption required in \citet{aminian2025kl}.\footnote{Technically, the $C_{\varepsilon_{\text{FKL}}}$ in \citet{aminian2025kl} requires the behavioral policy to cover every policy with a local forward-KL-ball of the reference policy well.} 
    It is also worth noting that this combination of an $\epsilon^{-1}$ fast rate and relaxed concentrability requirements has recently been shown to be both necessary and sufficient for reverse-KL-regularized CBs~\citep{zhao2026towards}.
\end{remark}

\color{black}

\section{Minimax Lower Bounds}\label{sec:lb}

In this section, we present worst-case hardness results in the tabular setting, whose proofs are deferred to \Cref{app:lb-proofs}.

\begin{theorem}\label{thm:fkl:mab:lb-ref=unif-A-arms:fast-rate:contextual}
Let $\hat{\pi} \in \Delta(\cA|\cS)$ be any estimator (with full support) from $n$ context-action-reward pairs, then for any even $A \geq 2$, $S\geq 1$, $\eta > 0$, and $n > 16SA \vee 4\eta^2 SA$,
\begin{align*}
 \sup_{\mathrm{CB}(\cS, \cA, r, \rho, \piref): |\cS| = S, |\cA| = A} \EE_{\cD \sim P_{\piref, r}} \suboptfkl(\hat\pi; \cS, \cA, r, \rho, \piref) \gtrsim  \frac{\eta SA}{n}.
\end{align*}
\end{theorem}
\Cref{thm:fkl:mab:lb-ref=unif-A-arms:fast-rate:contextual} implies a minimax sample complexity lower bound of $\Omega(\eta SA \epsilon^{-1})$ for learning an $\epsilon$-optimal policy in the tabular setting whenever $\epsilon \lesssim \eta \wedge \eta^{-1}$,\footnote{Note that $n > 16SA \vee 4\eta^2 SA$ is an applicability condition of \Cref{thm:fkl:mab:lb-ref=unif-A-arms:fast-rate:contextual} rather than an additional sample cost: when $\epsilon \lesssim \eta \wedge \eta^{-1}$, setting $n_0 = \lfloor 16SA \vee 4\eta^2 SA \rfloor + 1$ gives $\eta SA / n_0 \gtrsim \epsilon$, so by sample-discarding monotonicity no estimator with $n \leq n_0$ achieves expected suboptimality $\epsilon$, whence $n \gtrsim \eta SA \epsilon^{-1}$.} matching \Cref{thm:tabular} in the statistical rate $\epsilon^{-1}$ as well as the linear dependencies on $\eta$, $S$, and $A$ under uniform coverage ($C^{\pistar} = \Theta(1)$). Because our hard instances employ $\piref = \mathsf{Uniform}(\cA)$, characterizing the minimax optimal dependencies on the coverage coefficients—specifically $(C^{\pistar})^2$ in \Cref{thm:tabular} and $C^{\pistar} D_{\pistar}^2$ in \Cref{thm:function-approximation}—and deriving fast-rate lower bounds with function-class complexity under general function approximation remain open directions.

\noindent\textbf{The regime with large $\eta$.} Both our fast-rate upper and lower bounds scale linearly with $\eta$, mirroring \emph{reverse} KL~\citep{zhao2025sharp,zhao2026towards}. Note that \Cref{thm:fkl:mab:lb-ref=unif-A-arms:fast-rate:contextual} applies to the \emph{high-regularization} regime $\eta^2 \lesssim n/(SA)$. As $\eta$ grows large, the regularization vanishes and the problem enters a \emph{low-regularization} regime, where reverse-KL offline~\citep{zhao2026towards} and online~\citep{ji2026near} CBs transition to an $\eta$-free $\Omega(\epsilon^{-2})$ slow rate. In \Cref{subsec:fkl:two-arm} (\Cref{thm:fkl:mab:lb-ref=unif}), we prove an analogous $\Omega(\eta n^{-1} \wedge n^{-1/2})$ phase-transition lower bound for two-armed forward-KL bandits. Complementarily, \algcb{} simultaneously enjoys an $\eta$-independent $\tilde{O}(\epsilon^{-2})$ upper bound on the same high-probability event (see \Cref{app:subsec:slow-rate}), though matching the exact transition threshold of $\eta$ across all problem parameters ($C^{\pistar}, S, A, D_{\pistar}^2$) is left for future work.





\color{blue}

\color{black}


\section{Technical Overview}\label{sec:tech-overview}

We use $|\cS| = 1$ (offline bandits) as the working example to outline the limitations (or inapplicability) of previous works and highlight the key technique for our algorithm analysis.

\subsection{Limitations of Traditional Routines}

In the first statistical analysis \emph{specialized for} $\suboptfkl(\cdot)$, \citet[Theorem~6.3]{aminian2025kl} treat all terms involving forward KL with a gross bound depending on an all-policy-type concentrability coefficient $C_{\varepsilon_{\text{FKL}}}$ to reduce the suboptimality to its \emph{unregularized} counterpart, which unavoidably exhibits at least the fundamental limit of estimation, i.e., $\tilde{O}(\epsilon^{-2})$. Another conjugate-based routine~\citep{zhao2026towards} gives $\tilde{\Theta}(\epsilon^{-1})$ sample complexity offline CBs under general $f$-divergence-regularized objectives with strongly convex $f$, which excludes $\kl{\piref}{\pi}$. Recently, \citet{lee2026regularized} consider regularized objectives in an online setting for strongly convex regularizers in a unified manner, which is, however, not ideal for forward KL in the offline case because: although $\pi \mapsto \kl{\pi}{\piref}$ is $\Theta(1)$-strongly-convex with respect to the $L^1$ distance~\citep{polyanskiy2025information,zhao2026towards}, the modulus of strong convexity for $\pi \mapsto \kl{\piref}{\pi}$ may depend on the condition of $\piref$ as manifested in the following lemma, whose proof is deferred to \Cref{app:other-missing}.
\begin{lemma}\label{lem:fkl:sc}
Under \Cref{assump:full-supp}, if $\min_{a \in \cA} \piref(a) = \alpha$, $\pi \mapsto \kl{\piref}{\pi}$ is $\Theta(\alpha)$-strongly-convex with respect to the $L^1$ distance as a map from \textcolor{black}{$\ri\big(\Delta(\cA)\big)$} to $\RR_+$.
\end{lemma}
Consequently, techniques following \citet{lee2026regularized} for analyzing offline algorithms may yield sample complexity bounds \emph{proportional to} $\alpha^{-1}$, which translates to the undesirable (density-ratio-based) \emph{all}-policy concentrabilty even in multi-armed bandits.

\subsubsection{The Limitation of Techniques Tailored to \emph{Reverse} KL}

Existing results under relatively weak concentrability conditions for \emph{reverse}-KL-regularized objectives either heavily relies on specific log-linear parametric forms~\citep{foster2025good} or hinges on a curious pessimism-induced monotonic property in a Taylor-type argument based on the mean value theorem~\citep{zhao2026towards}.
In particular, we expose here the limitation in this latter mid-point-based routine, rendering it not sharp enough for forward KL. A faithful adaptation of the spirit in \citet[Section~2.4]{zhao2026towards} to the nature of $\kl{\piref}{\pi}$ and $\pistarfkl$ yields \Cref{eq:lem-perf-diff-full-supp,eq:lem-perf-diff-full-supp-coarse} in \Cref{lem:fkl:mab:perf-diff-fail}, whose proof is detailed in \Cref{app:other-missing}.

\begin{lemma}\label{lem:fkl:mab:perf-diff-fail}
If $\hat{\pi} \in \argmax_{\pi \in \Delta(\cA)} \objfkl(\pi; \hat{r})$, where $\hat{r}: \cA \to [0, 1]$; then $\exists {\bar{u}} \in [0, 1]$ such that under \Cref{assump:full-supp}, $\suboptfkl(\hat{\pi}) =$
\begin{align}
     \frac{\eta}{2}\bigg( \EE_{a \sim \pi_{\bar{u}}} \Big[ \Big( \frac{\pi_{\bar{u}}}{\piref}(r-\hat{r})^2 \Big)(a) \Big]- { \Big( \sum_{a} \big(\frac{\pi_{\bar{u}}^2}{\piref}(r-\hat{r})\big)(a) \Big)^2 } \Big/ { \sum_{a} \frac{\pi_{\bar{u}}^2}{\piref}(a) } \bigg) , \label{eq:lem-perf-diff-full-supp}
\end{align}
where $\pi_{\bar{u}} \in \argmax_{\pi \in \Delta(\cA)} \objfkl(\pi; r_{\bar{u}})$ with $r_{\bar{u}} \coloneqq (1-{\bar{u}})  r + {\bar{u}}  \hat{r}$. \Cref{eq:lem-perf-diff-full-supp} directly implies
\begin{align}
    \suboptfkl(\hat{\pi})
    \leq \frac{\eta}{2} \EE_{a \sim \pi_{\bar{u}}} \Big[ \Big( \frac{\pi_{\bar{u}}}{\piref}(r-\hat{r})^2 \Big)(a) \Big] . \label{eq:lem-perf-diff-full-supp-coarse}
\end{align}
\end{lemma}
The pillar behind the proof of \citet[Theorem~2.10]{zhao2026towards} is that the counterpart of \Cref{eq:lem-perf-diff-full-supp} or \Cref{eq:lem-perf-diff-full-supp-coarse} for \emph{reverse} KL is \emph{non-increasing} in $\bar{u}$ on $[0, 1]$ under \emph{the premise of pessimism} (i.e., $r \geq \hat{r}$), so that the maximizer of the corresponding RHS is $\bar{u} = 0$, which helps replace the troublemaker $\pi_{\bar{u}}$ with $\pistar$; and thus lets the RHS depend on the single-policy concentrability that only covers $\pistar$. Unfortunately, the calculations (deferred to \Cref{app:other-missing}) in \Cref{eq:lem:fkl:mab:perf-diff-fail--derivative} deny such a belief for forward KL.

\begin{lemma}\label{eq:lem:fkl:mab:perf-diff-fail--derivative}
   Let $F(\bar{u}) \coloneqq \text{RHS of \Cref{eq:lem-perf-diff-full-supp}}$, $w_u(a) \coloneqq \frac{\pi_u(a)}{\piref(a)}$, $Z_u \coloneqq \sum_{a} \frac{\pi_u^2(a)}{\piref(a)}$, and the probability distribution $\mu_u(a) \coloneqq \frac{\pi_u^2(a)}{\piref(a)}/ Z_u $, and $g \coloneqq r - \hat{r}$, then $\forall u \in (0,1)$,
\begin{align}
    F'(u) = -\eta^2 Z_u \cdot \EE_{a\sim \mu_u} \Big[ w_u(a) \big( g(a) - \EE_{a \sim \mu_u} g(a) \big)^3 \Big]. \label{eq:u-derivative-fkl}
\end{align}
\end{lemma}
The sign of \Cref{eq:u-derivative-fkl} may be hard to determine in general for $u \in (0, 1)$ even if $g \geq 0$ uniformly because it is a tilted skewness of a \emph{centered} random variable, which is strong evidence that the mean-value-type approach might not work for analyzing bandit learning under forward KL regularization.

\begin{remark}
    Sharp readers may wonder whether it is possible to disregard the second term in \Cref{eq:lem-perf-diff-full-supp} and take derivative of just \Cref{eq:lem-perf-diff-full-supp-coarse}. However, since \Cref{eq:lem-perf-diff-full-supp-coarse} is a part of \Cref{eq:lem-perf-diff-full-supp}, the calculations in \Cref{eq:lem:fkl:mab:perf-diff-fail--derivative} already imply that $\tilde{F}'(u) = -\eta^2 Z_u \EE_{a\sim \mu_u}\big[ w_u(a) g(a)^2\big( g(a) - \EE_{a\sim \mu_u}[g(a)] \big) \big]$ for RHS of \Cref{eq:lem-perf-diff-full-supp-coarse}$\eqqcolon\tilde{F}(\bar{u})$. Since $\tilde{F}'(u)$ turns out to be at best a \emph{plausible} tilted covariance between $g^2$ and a \emph{centered} $g - \EE_{\mu_u} g$, its sign may still be indefinite even conditioned on $g \geq 0$.
\end{remark}

\subsection{Our Key Technique}

Note that \Cref{eq:lem-perf-diff-full-supp} is already an \emph{identity}, which, together with the limitations we outline, suggests that it may not be suitable to tweak this mean-value argument based on Taylor expansion further. Instead, we prepare another identity \Cref{eq:lem:fkl:perf-diff-main-eq} based on a \emph{general} characterization (\Cref{lem:simplex-conjugate-bregman}) of \emph{regularized objectives} via the dual perspective of convex conjugate (over the simplex), which is \emph{fully detached from any specific design} of bonus terms; and craft a neat self-bounding argument \Cref{eq:fkl:perf-diff-self-bounding} on top of it to utilize the pessimism principle \emph{directly} in \Cref{eq:lem:fkl:perf-diff-direct-pess} without resorting to the aforementioned monotinicity mechanism. This suite of key techniques is elaborated in the proof of \Cref{lem:fkl:mab:perf-diff} as follows.

\begin{lemma}\label{lem:fkl:mab:perf-diff}
    Following the notation in \Cref{eq:fkl:offline-obj}, if $\hat{\pi} \in \argmax_{\pi \in \Delta(\cA)} \objfkl(\pi; \hat{r})$, where $\hat{r}: \cA \to [0, 1]$, and $r \geq \hat{r}$ uniformly; then under \Cref{assump:full-supp},
\begin{align}
    \suboptfkl(\hat{\pi}) = \objfkl(\pistarfkl; r) - \objfkl(\hat{\pi}; r) \leq 2 \eta \sum_{a \in \cA } \Big( \frac{\pistarfkl^2}{\piref}(r-\hat{r})^2 \Big)(a). \label{eq:fkl:mab:perf-diff-rhs}
\end{align}
\end{lemma}
\begin{proof}
Let $h(\cdot) \coloneqq \eta^{-1}\kl{\piref}{\cdot}$. Then $h$ is strictly convex by \Cref{lem:fkl:sc}, obviously continuously differentiable on $\ri(\Delta(\cA))$, $\mathrm{dom}(h) \coloneqq \{ \pi \in \Delta(\cA): h(\pi) < \infty\} = \ri(\Delta(\cA))$; and the uniqueness of $\pi_u$ as well as $\pi_u \in \ri(\Delta(\cA))$ is guaranteed by \Cref{lem:fkl:opt} and \Cref{assump:full-supp}. Thus, \Cref{lem:simplex-conjugate-duality,lem:simplex-conjugate-bregman} are applicable; under the notations of which: $\pistarfkl = \pi_r$ and $\hat{\pi} = \pi_{\hat{r}}$; $\objfkl(\pistarfkl; r) = h^*(r)$ and $\objfkl(\hat{\pi}; \hat{r}) = h^*(\hat{r})$ are their constrained convex conjugate, respectively.\footnote{\Cref{app:prop-conjugate-over-simplex} review concepts related to convex conjugate.} Therefore, \Cref{eq:lem:fkl:def-of-subopt}, which is inspired by \citet[Appendix~E.1]{zhao2026towards}, holds by definition; and $\nabla h^*(\hat{r}) = \hat{\pi}$ follows from \Cref{lem:simplex-conjugate-duality}. In detail,
\begin{align}
    \suboptfkl(\hat{\pi}) &= h^*(r) - h^*(\hat{r}) - \dotp{r}{\hat{\pi}} + \dotp{\hat{r}}{\hat{\pi}}  \label{eq:lem:fkl:def-of-subopt}\\
    &= \breg{h^*}{r}{\hat{r}}  = \breg{h}{\hat{\pi}}{\pistar} \label{eq:lem:fkl:value-based-perf-diff}\\
    &= \eta^{-1}\inner{\piref}{ \log \frac{\piref}{\hat{\pi}} } - \eta^{-1}\inner{\piref}{\log \frac{\piref}{\pistar} } -  \inner{\nabla h(\pistar)}{\hat{\pi} - \pistar} \notag\\
    &= \eta^{-1} \sum_{a \in \cA} \piref(a) \Big[ \frac{\hat{\pi}(a)}{\pistar(a)} - 1 -\log\Big( 1 - \big(1 - \frac{\hat{\pi}(a)}{\pistar(a)}\big) \Big) \Big] \notag\\
    &= \eta^{-1} \sum_{a \in \cA} \piref(a) \big[ -y_a - \log(1 - y_a) \big] = \eta^{-1} \sum_{a \in \cA} \piref(a) \phi(y_a), \label{eq:lem:fkl:perf-diff-main-eq}
\end{align}
where the first equality in \Cref{eq:lem:fkl:value-based-perf-diff} is a combination of the definition of Bregman divergence and $\nabla h^*(\hat{r}) = \hat{\pi}$; and the second equality in \Cref{eq:lem:fkl:value-based-perf-diff} is by applying \Cref{lem:simplex-conjugate-bregman} to $h$; $y_{a} \coloneqq 1 - {\hat{\pi}(a)} / {\pi^*(a)}, \forall a \in \cA$; and $\phi(y) \coloneqq -y - \log(1 - y), \forall y < 1$.
On the other hand, the definition of $\hat{\pi}$ alone already implies
\begin{align}
    \suboptfkl(\hat{\pi}) = \obj(\pistar; r) - \obj(\hat{\pi}; r)  &\leq \obj(\pistar; r) - \obj(\pistar; \hat{r}) + \obj(\hat{\pi}; \hat{r}) - \obj(\hat{\pi}; r)  \notag\\
    &= \dotp{r - \hat{r}}{\pistar} + \dotp{\hat{r} - r}{\hat{\pi}} \notag\\
    &= \dotp{r - \hat{r}}{\pistar - \hat{\pi}}  = \sum_{a \in \cA} \pistar(a) g(a) y_a,  \label{eq:lem:fkl:perf-diff-main-ineq}
\end{align}
where $g \coloneqq r - \hat{r}$. We thus define $z_a \coloneqq 2\eta \pistar(a)g(a) / \piref(a)$ to deduce from \Cref{eq:lem:fkl:perf-diff-main-ineq} that
\begin{align}
    2\cdot \suboptfkl(\hat{\pi}) &\leq \eta^{-1}\sum_{a \in \cA} \piref(a) z_a y_a \notag\\
    &\leq \eta^{-1} \sum_{a \in \cA} \piref(a) \big[ \phi(y_a) + \phi^*(z_a) \big] \notag\\
    &= \suboptfkl(\hat{\pi}) + \eta^{-1} \sum_{a \in \cA} \piref(a) \big( z_a - \log(1 + z_a) \big), \label{eq:fkl:perf-diff-self-bounding}
\end{align}
where the second inequality is due to the Fenchel-Young inequality, the last equality is due to \Cref{fact:fkl-phi-convex-conjugate}; and the pessimism $r \geq \hat{r}$ ensures $z_a \geq 0 > -1$, which in turn ensures the well-posedness of the constrained convex conjugate $\phi^*(z_a)$. Rearranging \Cref{eq:fkl:perf-diff-self-bounding} yields
\begin{align}
    \suboptfkl(\hat{\pi}) &\leq  \eta^{-1} \sum_{a \in \cA} \piref(a) \big( z_a - \log(1 + z_a) \big) \notag\\
    &\leq \eta^{-1} \sum_{a \in \cA} \piref(a) \frac{z_a^2}{2} \label{eq:lem:fkl:perf-diff-direct-pess} \\
    &= 2\eta \sum_{a \in \cA} \piref(a) \bigg( \frac{\pistar(a)}{\piref(a)} g(a)\bigg)^2, \notag
\end{align}
where the second inequality \Cref{eq:lem:fkl:perf-diff-direct-pess} is due to $z_a \geq 0$ (i.e., pessimism) together with the fact that $\forall z \geq 0, z - \log(1 + z) \leq z^2/2$; and the equality follows from the definition of $z_a$.
\end{proof}

Again, since the $\suboptfkl(\cdot)$ on general $\cS$ with $|\cS| \geq 1$ is the expectation of the bandit suboptimality over $s \sim \rho$, the sample complexity in \Cref{thm:tabular,thm:function-approximation} follow from the conjunction of a context-wise application of \Cref{lem:fkl:mab:perf-diff} and the corresponding concentration arguments conditioned on the pessimism events, respectively; as detailed in \Cref{app:ub-missing-proof}. To see why \Cref{lem:fkl:mab:perf-diff} helps bypass the need of all-policy concentrability, we observe that the RHS of \Cref{eq:fkl:mab:perf-diff-rhs} is bounded from above in the tabular setting by $2\eta (C^{\pistar})^2 \EE_{\piref}(r-\hat{r})^2$, where the squared loss admits a $\tilde{O}(n^{-1})$ concentration with high probability; the reasoning is similar under general function approximation.

\section{Conclusion, Limitations, and Future Work}

We initiate the study of forward-KL regularization as a performance metric for data efficiency in offline decision making. For offline CBs, we certify the sufficiency of single-policy concentrability for achieving $\tilde{O}(\epsilon^{-1})$ forward-KL-regularized sample complexity via \algcb{} in both the tabular and realizable function-approximation settings (while simultaneously retaining $\eta$-independent $\tilde{O}(\epsilon^{-2})$ slow rates; see \Cref{app:subsec:slow-rate}). Our lower bounds confirm rate-optimality in $\epsilon$ and $\eta$ (as well as $S$ and $A$ in the tabular setting) and demonstrate a fast-to-slow rate phase transition as regularization weakens.

Technically, while the simplex conjugate-Bregman identity (\Cref{lem:simplex-conjugate-bregman}) behind \Cref{lem:fkl:mab:perf-diff} holds for general strictly convex regularizers, achieving $\epsilon^{-1}$ fast rates relies on a decoupling property specific to $f(t) = -\log t$ and its synergy with offline pessimism; as we detail in \Cref{app:subsec:beyond-fkl}, general strictly convex $f$-divergences need not admit quadratic dual growth, and online optimism reverses the dual sign condition. Regarding \textbf{limitations} and open problems: (i) while dependence on $C^{\pistar}$ is natural under distribution shift, tightening the coverage dependencies, viz., $(C^{\pistar})^2$ in \Cref{thm:tabular} and $C^{\pistar}D_{\pistar}^2$ in \Cref{thm:function-approximation}, and deriving fast-rate lower bounds under general function approximation remain open; (ii) closing the parameter gap in the phase-transition threshold of $\eta$ beyond two-armed bandits (\Cref{thm:fkl:mab:lb-ref=unif}) is unsettled; and (iii) extending our guarantees beyond exact realizability (\Cref{assume:general-function-approx}) and one-step CBs to misspecified reward models and multi-step MDPs~\citep{jin2020provably,zhao2025logarithmic,lee2026regularized} is an important direction for future work.

\begin{ack}
We thank the anonymous reviewers for their helpful comments. QZ, KJ, HZ and QG are supported in part by the National Science Foundation DMS-2323113 and IIS-2403400. 
The views and conclusions contained in this paper are those of the authors and should not be interpreted as representing any funding agencies.
\end{ack}

\bibliographystyle{ims}
\bibliography{reference}

@inproceedings{jin2021pessimism,
  title={Is pessimism provably efficient for offline rl?},
  author={Jin, Ying and Yang, Zhuoran and Wang, Zhaoran},
  booktitle={International Conference on Machine Learning},
  pages={5084--5096},
  year={2021},
  organization={PMLR}
}

@article{xie2021bellman,
  title={Bellman-consistent pessimism for offline reinforcement learning},
  author={Xie, Tengyang and Cheng, Ching-An and Jiang, Nan and Mineiro, Paul and Agarwal, Alekh},
  journal={Advances in neural information processing systems},
  volume={34},
  pages={6683--6694},
  year={2021}
}

@article{rashidinejad2021bridging,
  title={Bridging offline reinforcement learning and imitation learning: A tale of pessimism},
  author={Rashidinejad, Paria and Zhu, Banghua and Ma, Cong and Jiao, Jiantao and Russell, Stuart},
  journal={Advances in Neural Information Processing Systems},
  volume={34},
  pages={11702--11716},
  year={2021}
}

@article{xie2021policy,
  title={Policy finetuning: Bridging sample-efficient offline and online reinforcement learning},
  author={Xie, Tengyang and Jiang, Nan and Wang, Huan and Xiong, Caiming and Bai, Yu},
  journal={Advances in neural information processing systems},
  volume={34},
  pages={27395--27407},
  year={2021}
}

@inproceedings{shi2022pessimistic,
  title={Pessimistic q-learning for offline reinforcement learning: Towards optimal sample complexity},
  author={Shi, Laixi and Li, Gen and Wei, Yuting and Chen, Yuxin and Chi, Yuejie},
  booktitle={International conference on machine learning},
  pages={19967--20025},
  year={2022},
  organization={PMLR}
}

@inproceedings{zhan2022offline,
  title={Offline reinforcement learning with realizability and single-policy concentrability},
  author={Zhan, Wenhao and Huang, Baihe and Huang, Audrey and Jiang, Nan and Lee, Jason},
  booktitle={Conference on Learning Theory},
  pages={2730--2775},
  year={2022},
  organization={PMLR}
}

@article{li2022pessimism,
  title={Pessimism for Offline Linear Contextual Bandits using $\ell_p$ Confidence Sets},
  author={Li, Gene and Ma, Cong and Srebro, Nati},
  journal={Advances in Neural Information Processing Systems},
  volume={35},
  pages={20974--20987},
  year={2022}
}

@article{xiong2022nearly,
  title={Nearly minimax optimal offline reinforcement learning with linear function approximation: Single-agent mdp and markov game},
  author={Xiong, Wei and Zhong, Han and Shi, Chengshuai and Shen, Cong and Wang, Liwei and Zhang, Tong},
  journal={arXiv preprint arXiv:2205.15512},
  year={2022}
}

@article{di2023pessimistic,
  title={Pessimistic nonlinear least-squares value iteration for offline reinforcement learning},
  author={Di, Qiwei and Zhao, Heyang and He, Jiafan and Gu, Quanquan},
  journal={arXiv preprint arXiv:2310.01380},
  year={2023}
}

@article{li2024settling,
  title={Settling the sample complexity of model-based offline reinforcement learning},
  author={Li, Gen and Shi, Laixi and Chen, Yuxin and Chi, Yuejie and Wei, Yuting},
  journal={The Annals of Statistics},
  volume={52},
  number={1},
  pages={233--260},
  year={2024},
  publisher={Institute of Mathematical Statistics}
}

@inproceedings{xiong2024iterative,
  title={Iterative preference learning from human feedback: Bridging theory and practice for rlhf under kl-constraint},
  author={Xiong, Wei and Dong, Hanze and Ye, Chenlu and Wang, Ziqi and Zhong, Han and Ji, Heng and Jiang, Nan and Zhang, Tong},
  booktitle={Forty-first International Conference on Machine Learning},
  year={2024}
}

@inproceedings{
    zhao2025sharp,
    title={Sharp Analysis for {KL}-Regularized Contextual Bandits and {RLHF}},
    author={Heyang Zhao and Chenlu Ye and Quanquan Gu and Tong Zhang},
    booktitle={The Thirty-ninth Annual Conference on Neural Information Processing Systems},
    year={2025}
}

@article{kim2026coverage,
  title={Coverage Improvement and Fast Convergence of On-policy Preference Learning},
  author={Kim, Juno and Yun, Jihun and Lee, Jason D and Jun, Kwang-Sung},
  journal={arXiv preprint arXiv:2601.08421},
  year={2026}
}

@article{farahmand2010error,
  title={Error propagation for approximate policy and value iteration},
  author={Farahmand, Amir-massoud and Szepesv{\'a}ri, Csaba and Munos, R{\'e}mi},
  journal={Advances in neural information processing systems},
  volume={23},
  year={2010}
}

@inproceedings{chen2019information,
  title={Information-theoretic considerations in batch reinforcement learning},
  author={Chen, Jinglin and Jiang, Nan},
  booktitle={International Conference on Machine Learning},
  pages={1042--1051},
  year={2019},
  organization={PMLR}
}

@article{uehara2021pessimistic,
  title={Pessimistic model-based offline reinforcement learning under partial coverage},
  author={Uehara, Masatoshi and Sun, Wen},
  journal={arXiv preprint arXiv:2107.06226},
  year={2021}
}

@inproceedings{zhu2023principled,
  title={Principled reinforcement learning with human feedback from pairwise or k-wise comparisons},
  author={Zhu, Banghua and Jordan, Michael and Jiao, Jiantao},
  booktitle={International Conference on Machine Learning},
  pages={43037--43067},
  year={2023},
  organization={PMLR}
}

@article{zhan2023provable,
  title={Provable offline preference-based reinforcement learning},
  author={Zhan, Wenhao and Uehara, Masatoshi and Kallus, Nathan and Lee, Jason D and Sun, Wen},
  journal={arXiv preprint arXiv:2305.14816},
  year={2023}
}

@article{huang2024correcting,
  title={Correcting the Mythos of KL-Regularization: Direct Alignment without Overoptimization via Chi-Squared Preference Optimization},
  author={Huang, Audrey and Zhan, Wenhao and Xie, Tengyang and Lee, Jason D and Sun, Wen and Krishnamurthy, Akshay and Foster, Dylan J},
  journal={arXiv preprint arXiv:2407.13399v3},
  year={2024}
}

@article{schulman2015trust,
  title={Trust Region Policy Optimization},
  author={Schulman, John},
  journal={arXiv preprint arXiv:1502.05477},
  year={2015}
}

@inproceedings{haarnoja2018soft,
  title={Soft actor-critic: Off-policy maximum entropy deep reinforcement learning with a stochastic actor},
  author={Haarnoja, Tuomas and Zhou, Aurick and Abbeel, Pieter and Levine, Sergey},
  booktitle={International conference on machine learning},
  pages={1861--1870},
  year={2018},
  organization={PMLR}
}

@article{rafailov2023direct,
  title={Direct preference optimization: Your language model is secretly a reward model},
  author={Rafailov, Rafael and Sharma, Archit and Mitchell, Eric and Manning, Christopher D and Ermon, Stefano and Finn, Chelsea},
  journal={Advances in Neural Information Processing Systems},
  volume={36},
  year={2023}
}

@article{xie2024exploratory,
  title={Exploratory Preference Optimization: Harnessing Implicit Q*-Approximation for Sample-Efficient RLHF},
  author={Xie, Tengyang and Foster, Dylan J and Krishnamurthy, Akshay and Rosset, Corby and Awadallah, Ahmed and Rakhlin, Alexander},
  journal={arXiv preprint arXiv:2405.21046},
  year={2024}
}

@article{williams1992simple,
  title={Simple statistical gradient-following algorithms for connectionist reinforcement learning},
  author={Williams, Ronald J},
  journal={Machine learning},
  volume={8},
  pages={229--256},
  year={1992},
  publisher={Springer}
}

@inproceedings{ziebart2008maximum,
  title={Maximum entropy inverse reinforcement learning.},
  author={Ziebart, Brian D and Maas, Andrew L and Bagnell, J Andrew and Dey, Anind K and others},
  booktitle={Aaai},
  volume={8},
  pages={1433--1438},
  year={2008},
  organization={Chicago, IL, USA}
}

@inproceedings{levine2013guided,
  title={Guided policy search},
  author={Levine, Sergey and Koltun, Vladlen},
  booktitle={International conference on machine learning},
  pages={1--9},
  year={2013},
  organization={PMLR}
}

@article{ouyang2022training,
  title={Training language models to follow instructions with human feedback},
  author={Ouyang, Long and Wu, Jeffrey and Jiang, Xu and Almeida, Diogo and Wainwright, Carroll and Mishkin, Pamela and Zhang, Chong and Agarwal, Sandhini and Slama, Katarina and Ray, Alex and others},
  journal={Advances in neural information processing systems},
  volume={35},
  pages={27730--27744},
  year={2022}
}

@inproceedings{jin2020provably,
  title={Provably efficient reinforcement learning with linear function approximation},
  author={Jin, Chi and Yang, Zhuoran and Wang, Zhaoran and Jordan, Michael I},
  booktitle={Conference on learning theory},
  pages={2137--2143},
  year={2020},
  organization={PMLR}
}

@book{lattimore2020bandit,
  title={Bandit algorithms},
  author={Lattimore, Tor and Szepesv{\'a}ri, Csaba},
  year={2020},
  publisher={Cambridge University Press}
}

@book{zhang2023ltbook,
   title={Mathematical Analysis of Machine Learning Algorithms},
   author={Zhang, Tong},
   doi={10.1017/9781009093057},
   publisher={Cambridge University Press},
   place={Cambridge},
   year={2023}
}

@article{foster2023foundations,
  title={Foundations of reinforcement learning and interactive decision making},
  author={Foster, Dylan J and Rakhlin, Alexander},
  journal={arXiv preprint arXiv:2312.16730},
  year={2023}
}

@book{wainwright2019high,
  title={High-dimensional statistics: A non-asymptotic viewpoint},
  author={Wainwright, Martin J},
  volume={48},
  year={2019},
  publisher={Cambridge university press}
}

@article{lecam1973convergence,
  title={Convergence of estimates under dimensionality restrictions},
  author={Le Cam, Lucien},
  journal={The Annals of Statistics},
  pages={38--53},
  year={1973},
  publisher={JSTOR}
}

@incollection{yu1997assouad,
  title={Assouad, fano, and le cam},
  author={Yu, Bin},
  booktitle={Festschrift for Lucien Le Cam: research papers in probability and statistics},
  pages={423--435},
  year={1997},
  publisher={Springer}
}

@book{polyanskiy2025information,
  title={Information theory: From coding to learning},
  author={Polyanskiy, Yury and Wu, Yihong},
  year={2025},
  publisher={Cambridge university press}
}

@article{chen2024assouad,
  title={Assouad, Fano, and Le Cam with Interaction: A Unifying Lower Bound Framework and Characterization for Bandit Learnability},
  author={Chen, Fan and Foster, Dylan J and Han, Yanjun and Qian, Jian and Rakhlin, Alexander and Xu, Yunbei},
  journal={Advances in Neural Information Processing Systems},
  volume={37},
  pages={75585--75641},
  year={2024}
}

@book{beck2017first,
  title={First-order methods in optimization},
  author={Beck, Amir},
  year={2017},
  publisher={SIAM}
}

@article{foster2025good,
  title={Is a Good Foundation Necessary for Efficient Reinforcement Learning? The Computational Role of the Base Model in Exploration},
  author={Foster, Dylan J and Mhammedi, Zakaria and Rohatgi, Dhruv},
  journal={arXiv preprint arXiv:2503.07453},
  year={2025}
}

@book{bertsekas2003convex,
  title={Convex analysis and optimization},
  author={Bertsekas, Dimitri and Nedic, Angelia and Ozdaglar, Asuman},
  volume={1},
  year={2003},
  publisher={Athena Scientific}
}

@article{milgrom2002envelope,
  title={Envelope theorems for arbitrary choice sets},
  author={Milgrom, Paul and Segal, Ilya},
  journal={Econometrica},
  volume={70},
  number={2},
  pages={583--601},
  year={2002},
  publisher={Wiley Online Library}
}

@inproceedings{zhao2026towards,
    title={Towards a Sharp Analysis of Offline Policy Learning for \$f\$-Divergence-Regularized Contextual Bandits},
    author={Qingyue Zhao and Kaixuan Ji and Heyang Zhao and Tong Zhang and Quanquan Gu},
    booktitle={The Fourteenth International Conference on Learning Representations},
    year={2026}
}

@book{rockafellar1997convex,
  title={Convex analysis},
  author={Rockafellar, R Tyrrell},
  volume={28},
  year={1997},
  publisher={Princeton university press}
}

@inproceedings{aminian2025kl,
  title={KL-Regularized RLHF with Multiple Reference Models: Exact Solutions and Sample Complexity},
  author={Aminian, Gholamali and Asadi, Amir R and Shenfeld, Idan and Mroueh, Youssef},
  booktitle={The Thirty-ninth Annual Conference on Neural Information Processing Systems},
  year={2025}
}

@article{lee2026regularized,
  title={Regularized Online RLHF with Generalized Bilinear Preferences},
  author={Lee, Junghyun and Hong, Minju and Jun, Kwang-Sung and Yun, Chulhee and Yun, Se-Young},
  journal={arXiv preprint arXiv:2602.23116},
  year={2026}
}

@article{nayak2025achieving,
  title={Achieving Logarithmic Regret in KL-Regularized Zero-Sum Markov Games},
  author={Nayak, Anupam and Yang, Tong and Yagan, Osman and Joshi, Gauri and Chi, Yuejie},
  journal={arXiv preprint arXiv:2510.13060},
  year={2025}
}

@article{zhang2026beyond,
  title={Beyond Pessimism: Offline Learning in KL-regularized Games},
  author={Zhang, Yuheng and Chen, Claire and Jiang, Nan},
  journal={arXiv preprint arXiv:2604.06738},
  year={2026}
}

@article{ji2026near,
  title={Near-Optimal Regret for KL-Regularized Multi-Armed Bandits},
  author={Ji, Kaixuan and Zhao, Qingyue and Zhao, Heyang and Di, Qiwei and Gu, Quanquan},
  journal={arXiv preprint arXiv:2603.02155},
  year={2026}
}

@inproceedings{zhao2025logarithmic,
  title={Logarithmic Regret for Online KL-Regularized Reinforcement Learning},
  author={Zhao, Heyang and Ye, Chenlu and Xiong, Wei and Gu, Quanquan and Zhang, Tong},
  booktitle={International Conference on Machine Learning},
  pages={77864--77884},
  year={2025},
  organization={PMLR}
}

@inproceedings{tiapkin2023fast,
  title={Fast rates for maximum entropy exploration},
  author={Tiapkin, Daniil and Belomestny, Denis and Calandriello, Daniele and Moulines, Eric and Munos, Remi and Naumov, Alexey and Perrault, Pierre and Tang, Yunhao and Valko, Michal and Menard, Pierre},
  booktitle={International Conference on Machine Learning},
  pages={34161--34221},
  year={2023},
  organization={PMLR}
}

@article{wu2025greedy,
  title={Greedy Sampling Is Provably Efficient for RLHF},
  author={Wu, Di and Shi, Chengshuai and Yang, Jing and Shen, Cong},
  journal={arXiv preprint arXiv:2510.24700},
  year={2025}
}

@article{gx2025kl,
  title={KL-Regularized Reinforcement Learning is Designed to Mode Collapse},
  author={GX-Chen, Anthony and Prakash, Jatin and Guo, Jeff and Fergus, Rob and Ranganath, Rajesh},
  journal={arXiv preprint arXiv:2510.20817},
  year={2025}
}

@article{wang2023beyond,
  title={Beyond reverse kl: Generalizing direct preference optimization with diverse divergence constraints},
  author={Wang, Chaoqi and Jiang, Yibo and Yang, Chenghao and Liu, Han and Chen, Yuxin},
  journal={arXiv preprint arXiv:2309.16240},
  year={2023}
}

@article{tang2025few,
  title={On a few pitfalls in kl divergence gradient estimation for rl},
  author={Tang, Yunhao and Munos, R{\'e}mi},
  journal={arXiv preprint arXiv:2506.09477},
  year={2025}
}

@article{shah2025comedy,
  title={A Comedy of Estimators: On KL Regularization in RL Training of LLMs},
  author={Shah, Vedant and Obando-Ceron, Johan and Jain, Vineet and Bartoldson, Brian and Kailkhura, Bhavya and Mittal, Sarthak and Berseth, Glen and Castro, Pablo Samuel and Bengio, Yoshua and Malkin, Nikolay and others},
  journal={arXiv preprint arXiv:2512.21852},
  year={2025}
}

@article{zhang2026ema,
  title={EMA Policy Gradient: Taming Reinforcement Learning for LLMs with EMA Anchor and Top-k KL},
  author={Zhang, Lunjun and Ba, Jimmy},
  journal={arXiv preprint arXiv:2602.04417},
  year={2026}
}

@inproceedings{shan2025forward,
  title={Forward kl regularized preference optimization for aligning diffusion policies},
  author={Shan, Zhao and Fan, Chenyou and Qiu, Shuang and Shi, Jiyuan and Bai, Chenjia},
  booktitle={Proceedings of the AAAI Conference on Artificial Intelligence},
  volume={39},
  number={13},
  pages={14386--14395},
  year={2025}
}

@book{bishop2006pattern,
  title={Pattern recognition and machine learning},
  author={Bishop, Christopher M and Nasrabadi, Nasser M},
  volume={4},
  number={4},
  year={2006},
  publisher={Springer}
}

@article{ji2023language,
  title={Language model decoding as direct metrics optimization},
  author={Ji, Haozhe and Ke, Pei and Wang, Hongning and Huang, Minlie},
  journal={arXiv preprint arXiv:2310.01041},
  year={2023}
}

@article{guo2025deepseek,
  title={Deepseek-r1: Incentivizing reasoning capability in llms via reinforcement learning},
  author={Guo, Daya and Yang, Dejian and Zhang, Haowei and Song, Junxiao and Wang, Peiyi and Zhu, Qihao and Xu, Runxin and Zhang, Ruoyu and Ma, Shirong and Bi, Xiao and others},
  journal={arXiv preprint arXiv:2501.12948},
  year={2025}
}

@inproceedings{mcallester1999pac,
  title={PAC-Bayesian model averaging},
  author={McAllester, David A},
  booktitle={Proceedings of the twelfth annual conference on Computational learning theory},
  pages={164--170},
  year={1999}
}

@inproceedings{gentile2022achieving,
  title={Achieving minimax rates in pool-based batch active learning},
  author={Gentile, Claudio and Wang, Zhilei and Zhang, Tong},
  booktitle={International Conference on Machine Learning},
  pages={7339--7367},
  year={2022},
  organization={PMLR}
}

@inproceedings{agarwal2023vo,
  title={VO $ Q $ L: Towards Optimal Regret in Model-free RL with Nonlinear Function Approximation},
  author={Agarwal, Alekh and Jin, Yujia and Zhang, Tong},
  booktitle={The Thirty Sixth Annual Conference on Learning Theory},
  pages={987--1063},
  year={2023},
  organization={PMLR}
}

@article{yin2021towards,
  title={Towards instance-optimal offline reinforcement learning with pessimism},
  author={Yin, Ming and Wang, Yu-Xiang},
  journal={Advances in neural information processing systems},
  volume={34},
  pages={4065--4078},
  year={2021}
}

@article{wang2022gap,
  title={On gap-dependent bounds for offline reinforcement learning},
  author={Wang, Xinqi and Cui, Qiwen and Du, Simon},
  journal={Advances in Neural Information Processing Systems},
  volume={35},
  pages={14865--14877},
  year={2022}
}

@article{zanette2021provable,
  title={Provable benefits of actor-critic methods for offline reinforcement learning},
  author={Zanette, Andrea and Wainwright, Martin J and Brunskill, Emma},
  journal={Advances in neural information processing systems},
  volume={34},
  pages={13626--13640},
  year={2021}
}

@inproceedings{scherrer2014approximate,
  title={Approximate policy iteration schemes: A comparison},
  author={Scherrer, Bruno},
  booktitle={International Conference on Machine Learning},
  pages={1314--1322},
  year={2014},
  organization={PMLR}
}

@inproceedings{xie2021batch,
  title={Batch value-function approximation with only realizability},
  author={Xie, Tengyang and Jiang, Nan},
  booktitle={International Conference on Machine Learning},
  pages={11404--11413},
  year={2021},
  organization={PMLR}
}


\appendix

\section*{Additional Notations in Appendix}

Boldface capital letters are reserved for matrices. Boldface lowercase letters denote vectors in Euclidean spaces or vector-valued functions. For two probability distributions $P$ and $Q$, the total variation (TV) distance is denoted by $\tv{P}{Q} \coloneqq 0.5\int|\ud P - \ud Q|$. $\forall n \in \NN_{++}$, $\overline{[n]} \coloneqq \{0, 1, .., n-1\}$. For dimension-compatible vectors $\xb, \yb$, $\xb \odot \yb$ denotes their element-wise product.


\section{Additional Related Works}\label{app:add-related}

\noindent\textbf{Pessimism in Offline RL.} Classical analyses of batch value and policy iteration~\citep{farahmand2010error,scherrer2014approximate,chen2019information,xie2021batch} establish sample complexity guarantees under all-policy (uniform) concentrability, requiring the data distribution to cover all candidate policies or value functions so as to prevent out-of-distribution overestimation by greedy updates. To relax this uniform coverage condition, a substantial line of work incorporates pessimism in the face of uncertainty, penalizing uncertain state-action values or optimizing over conservative confidence sets so that coverage is required only with respect to a single comparator policy $\pistar$. Existing guarantees can be grouped into tabular and function-approximation regimes. In tabular offline CBs and MDPs, lower-confidence-bound (LCB) penalties in value iteration, Q-learning, and model-based planning attain minimax-optimal $\tilde{O}(\epsilon^{-2})$ sample complexity under density-ratio single-policy concentrability $C^{\pistar}$~\citep{rashidinejad2021bridging,shi2022pessimistic,li2024settling}, alongside variance- and gap-dependent refinements~\citep{yin2021towards,wang2022gap}. Under linear and general function approximation, explicit pointwise LCB bonuses~\citep{jin2021pessimism,xiong2022nearly,li2022pessimism,di2023pessimistic} as well as Bellman-consistent and model-based pessimistic formulations~\citep{zanette2021provable,xie2021bellman,uehara2021pessimistic,zhan2022offline}---and their counterparts in offline preference-based alignment~\citep{zhu2023principled,zhan2023provable}---replace tabular density ratios with feature-covariance or Bellman-evaluation concentrability coefficients while retaining $\tilde{O}(\epsilon^{-2})$ worst-case rates for unregularized objectives. For reverse-KL-regularized offline learning, plug-in estimators without pessimism achieve $\tilde{O}(\epsilon^{-1})$ fast rates either under all-policy concentrability~\citep{zhao2025sharp} or with $\exp(O(\eta))$ dependence under reference-policy coverage~\citep{wu2025greedy}. While strongly convex $f$-divergences admit coverage-free $\tilde{O}(\epsilon^{-1})$ rates via plug-in estimation and dual-Bregman smoothness, \citet{zhao2026towards} combine LCB reward estimation with a mean-value monotonicity argument to attain polynomial-in-$\eta$, $\tilde{O}(\epsilon^{-1})$ sample complexity under single-policy concentrability for reverse KL. Because forward KL ($f(x) = -\log x$) lacks uniform strong convexity on $\Delta(\cA)$ and invalidates this mean-value monotonicity (\Cref{lem:fkl:sc,eq:lem:fkl:mab:perf-diff-fail--derivative}), our analysis invokes pessimism directly to guarantee non-negativity of the dual argument in a constrained Fenchel--Young self-bounding inequality (\Cref{lem:fkl:mab:perf-diff}), yielding $\tilde{O}(\epsilon^{-1})$ sample complexity under single-policy concentrability.

\noindent\textbf{Forward-KL Regularization in RL.} Forward-KL regularization has been recently utilized in variants of Direct Preference Optimization (DPO)~\citep{rafailov2023direct}. \citet{wang2023beyond} generalized DPO to $f$-divergence penalties, which encompasses forward KL; while other specific alignment algorithms formulated via forward-KL regularization have also been developed for large language models (LLMs) \citep{ji2023language} and diffusion models \citep{shan2025forward}. Another line of research investigates how practical gradient estimation of KL divergence in RL implicitly induces forward-KL regularization. \citet{tang2025few} and \citet{shah2025comedy} theoretically and empirically demonstrated that auto-differentiating certain variance-reduced Monte Carlo estimators of reverse KL in some famous KL-regularized RL fine-tuning algorithms actually produces gradients of the forward KL in expectation. Furthermore, \citet{zhang2026ema} introduced an policy gradient framework that incorporates top-$k$ forward-KL regularization to stabilize LLM fine-tuning. A traditional viewpoint on the distinct divergence properties of forward and reverse KL, which has usually been respectively characterized as mass-covering and mode-seeking, is systematically ablated by \citet{gx2025kl}, who claim that diversity collapse can be an inherent consequence of the regularized RL objective itself regardless of the KL direction. On the theoretical front, the sample complexity of forward-KL-regularized RL is largely unclear. Recently, \citet{aminian2025kl} provided the first $\tilde{O}(\epsilon^{-2})$ slow-rate upper bound for offline contextual dueling bandits with multiple reference models under forward-KL regularization, leaving the achievability of $\epsilon^{-1}$-type fast rates as an open question.


\section{Relation between Coverage Notions}\label{app:discussion-coverage}

In this section, we provide more illustrations on the relation between two coverage measures, $D^2_{\pistar}$ and $C^{\pistar}$, in the context of forward-KL regularized bandits. In particular, we provide two cases under linear function approximation, in one of which $D^2_{\pistar} = \Omega(d C^{\pi^*})$ and in the other we have $D^2_{\pistar} \ll C^{\pi^*}$, where $d$ is the dimension of the function class. We summarized them as two propositions.

\begin{proposition}
For any $C>1$, $d \ge 3$ being odd, and $\eta \geq 2$, there exist a forward-KL-regularized linear bandit instance, such that $C^{\pi^*}=O(C)$ and $D^2_{\pistar} = \Omega(Cd)$, leading to $D^2_{\pistar} = \Omega(d C^{\pi^*})$.
\end{proposition}

\begin{proof}
We construct the instance as follows. Let $d = 2A+1$ be some odd number and consider an $2A+1$-armed bandit, such that the feature vector of the $i$-th arm, $\bphi(a_i) = \eb_i \in \RR^d$, which has $1$ on its $i$-th entry and $0$ on all other entries. The reference policy $\piref(a_i) = (2AC)^{-1}$ for $i \in [2A]$ and $\piref(a_{2A+1}) = (C-1)/C$, where $C \geq 1$ and $\eta \geq 2$. The ground truth reward function $\btheta^* = \sum_{i \leq A}\eb_i$ and the function class is given by all $\|\btheta\|_{\infty} \leq 1$. By construction, we know that $\pi^*(a_i) \geq \piref(a_i)$ if and only if $i \in [A]$ and its closed form is given by
\begin{align*}
    \pi^*(a_i) = \frac{1}{A} \frac{\sqrt{(\eta - 1)^2C^2 + 2\eta C} + (\eta - 1)C}{2\eta C} \leq \frac{2}{A}
\end{align*}
which gives $C^{\pi^*} = O(C)$. Now we compute the $D^2_{\pi^*}$ of this instance. For all $i \in [A]$, we know that
\begin{align*}
    D^2(a_i) = \sup_{\|\btheta\|_{\infty} \leq 2} \frac{\la \btheta, \eb_i \ra^2}{\EE_{\piref}\la \btheta, \eb_j \ra^2} = 2CA = \Theta(Cd),
\end{align*}
where the second equation holds with $\btheta = \eb_i$. Taking expectation over $\pistar$, we have
\begin{align*}
    D^2_{\pistar} \geq \frac{1}{2A}\sum_{i \in [A]}D^2(a_i) = \Theta(C^{\pistar}d),
\end{align*}
which concludes the proof.
\end{proof}

The following proposition provides another instance on which $D^2_{\pistar} \ll C^{\pistar}$.

\begin{proposition}
For any $C \geq 2$ and $\eta \geq 2$, there exists a KL-regularized linear bandit instance, such that $C^{\pistar} = C/2$ and $D^2_{\pistar} = \Theta(1)$.
\end{proposition}

\begin{proof}
We consider the function class of $\btheta \in \RR^2$ and $\|\btheta\| \leq \sqrt{2}$. The instance consists of three arms, where $\bphi(a_1) = (1,0)$, $\bphi(a_2) = (0,1)$, and $\bphi(a_3) = (1,1)$. The ground truth parameter $\btheta^* = (1, 1)$. The reference policy is given by $\piref(a_1) = \piref(a_2) = 1/2 - 1/2C$ and $\piref(a_3) = 1/C$, where $C \geq 2$. We further fix any $\eta \geq 2$. A direct computation yields that 
\begin{align*}
    \pi^*(a_3) = \frac{(\eta -1) + \sqrt{(\eta - 1)^2 + 4\eta C^{-1}}}{2\eta} \geq \frac{\eta-1}{\eta}, 
\end{align*}
which results in $C^{\pi^*} \geq C/2$. On the other hand, we know that for $i=1,2$, we have
$D^2(a_i) \leq \piref(a_i)^{-1} \leq 4$. As for $a_3$, since we have $\la \btheta, \phi(a_3) \ra^2 = \la \btheta, \phi(a_1) + \phi(a_2) \ra^2 \leq 2 \la \btheta, \phi(a_1)\ra^2 + 2\la \btheta, \phi(a_2)\ra$, which gives that $D^2(a_3) \leq 2D^2(a_1) + 2D^2(a_2) \leq 16$. Therefore, taking expectation over $\pi^*$, we know that $D^2_{\pistar} \leq 16$ which is a constant.
\end{proof}

\color{black}

\section{Missing Proofs in \Cref{sec:ub}}\label{app:ub-missing-proof}

\subsection{Proof of Lemma~\ref{lem:pessimistic-tabular}}

\begin{proof}[Proof of Lemma~\ref{lem:pessimistic-tabular}]
Let $S \coloneqq |\cS|$ and $A \coloneqq |\cA|$. We first show that $\confbandittb(\delta)$ holds with probability at least $1 - \delta/2$. Fix any $(s,a) \in \cS \times \cA$ and condition on the realized context-action pairs $\{(s_i, a_i)\}_{i=1}^n$. When $N(s,a)=0$, \Cref{algorithm:bandit-pess-tb} sets $\fls(s,a) = 0$ and $b(s,a) = 1$; since $\fgt(s,a) \in [0, 1]$, $\big|\fls(s,a) - \fgt(s,a)\big| = \fgt(s,a) \leq 1 = b(s,a)$ holds deterministically. When $N(s,a) = m \geq 1$, we have $\fls(s,a) - \fgt(s,a) = \frac{1}{m}\sum_{i: (s_i,a_i)=(s,a)} \varepsilon_i$; conditioned on $\{(s_i, a_i)\}_{i=1}^n$, the noises $\{\varepsilon_i : (s_i,a_i)=(s,a)\}$ are mutually independent, mean-zero, and $1$-subgaussian. Applying the subgaussian Hoeffding inequality (\Cref{lem:azuma-hoeffding}) conditioned on $\{(s_i, a_i)\}_{i=1}^n$ gives
\begin{align*}
    \PP\Big( \big|\fls(s,a) - \fgt(s,a)\big| > b(s,a) \,\Big|\, \{(s_i, a_i)\}_{i=1}^n \Big) \leq 2\exp\Big( -\frac{m b(s,a)^2}{2} \Big) = 2\Big(\frac{\delta}{2SA}\Big)^2 \leq \frac{\delta}{2SA},
\end{align*}
where the last inequality uses $\delta \in (0, 1)$ and $SA \geq 1$. Taking expectation over $\{(s_i, a_i)\}_{i=1}^n$ and applying a union bound over all $(s,a) \in \cS \times \cA$, we obtain $\PP\big(\confbandittb(\delta)\big) \geq 1 - \delta/2$. Next, by \Cref{lem:binary-concentration}, for each fixed $(s,a) \in \cS \times \cA$, with probability at least $1-\delta/(2SA)$,
\begin{align*}
    \frac{1}{N(s,a) \vee 1} \leq \frac{8\log(2SA/\delta)}{n \rho(s) \piref(a|s)}.
\end{align*}
Taking a union bound over all $(s,a)\in\cS\times\cA$ shows $\PP\big(\confbanditbin(\delta)\big) \geq 1-\delta/2$. A final union bound over $\confbandittb(\delta)$ and $\confbanditbin(\delta)$ completes the proof.
\end{proof}

\subsection{Proof of Theorem~\ref{thm:function-approximation}}

\begin{proof}[Proof of Theorem~\ref{thm:function-approximation}]
On event $\confbanditfa(\delta)$, we have $|\fls(s,a) - \fgt(s,a)| \leq \Gamma_n(s,a)$ for all $(s,a) \in \cS \times \cA$. Since $\fgt(s,a) \in [0, 1]$ and $\fps(s,a) = \max\{\fls(s,a) - \Gamma_n(s,a), 0\}$, it holds for all $(s,a) \in \cS \times \cA$ that $\fps(s,a) \in [0, 1]$ and
\begin{align*}
    0 \leq \fgt(s,a) - \fps(s,a) \leq \fgt(s,a) - \big(\fls(s,a) - \Gamma_n(s,a)\big) \leq 2\Gamma_n(s,a).
\end{align*}
Therefore, invoking Lemma~\ref{lem:fkl:mab:perf-diff} context-wise and taking expectation over $s \sim \rho$ gives
\begin{align*}
    \suboptfkl(\hat{\pi}) &\leq 2 \eta \EE_{s \sim \rho} \Bigg[ \sum_{a \in \cA } \bigg( \frac{\pistarfkl^2(a|s)}{\piref(a|s)}(\fgt(s,a) - \fps(s,a))^2 \bigg) \Bigg] \\
    & = 2 \eta \EE_{s \sim \rho} \Bigg[ \sum_{a \in \cA } \bigg[ \bigg(\frac{\pistarfkl(a|s)}{\piref(a|s)}\bigg)\pistarfkl(a|s)(\fgt(s,a) - \fps(s,a))^2 \bigg] \Bigg] \\
    & \leq 8 \eta C^{\pi^*} \EE_{(s,a) \sim \rho\times \pistarfkl} \big[\Gamma_n^2(s,a) \big] \\
    & = 8 \eta C^{\pi^*} \beta^2 D_{\pi^*}^2,
\end{align*}
where the second inequality uses $\pistarfkl(a|s)/\piref(a|s) \leq C^{\pi^*}$ and $\fgt - \fps \leq 2\Gamma_n$, and the last equality uses \Cref{eq:bandit-pess-function-approximation} and \Cref{assume:single-coverage-bandit}. Setting the covering scale in $\beta$ to $\epsilon_c \asymp n^{-1}$ under \Cref{assume:poly-covering} yields $\suboptfkl(\hat{\pi}) \leq \tilde{O}\big(\eta C^{\pi^*} D_{\pi^*}^2 n^{-1} \log(N_{\fcl}(n^{-1})/\delta)\big)$, which translates to the stated sample complexity. Finally, by Lemma~\ref{lem:pessimistic-function-approximation}, $\confbanditfa(\delta)$ holds with probability at least $1-\delta$.
\end{proof}

\subsection{Proof of Theorem~\ref{thm:tabular}}

\begin{proof}[Proof of Theorem~\ref{thm:tabular}]
On event $\confbandittb(\delta) \cap \confbanditbin(\delta)$, fix any $(s, a) \in \cS \times \cA$. When $N(s,a) \geq 1$, $|\fls(s,a) - \fgt(s,a)| \leq b(s,a)$ and $\fgt(s,a) \in [0, 1]$ imply $\fps(s,a) = \max\{\fls(s,a) - b(s,a), 0\} \in [0, 1]$ and $0 \leq \fgt(s,a) - \fps(s,a) \leq 2b(s,a)$, whence $(\fgt(s,a) - \fps(s,a))^2 \leq 4b(s,a)^2 \leq \frac{128\log^2(2|\cS||\cA|/\delta)}{n\rho(s)\piref(a|s)}$ by $\confbanditbin(\delta)$. When $N(s,a) = 0$, \Cref{algorithm:bandit-pess-tb} sets $\fps(s,a) = 0 \in [0, 1]$, so $0 \leq \fgt(s,a) - \fps(s,a) = \fgt(s,a) \leq 1$, and $\confbanditbin(\delta)$ yields $(\fgt(s,a) - \fps(s,a))^2 \leq \frac{1}{N(s,a) \vee 1} \leq \frac{8\log(2|\cS||\cA|/\delta)}{n\rho(s)\piref(a|s)}$. Consequently, $\fps(s,a) \in [0, 1]$ and $0 \leq \fgt(s,a) - \fps(s,a)$ hold uniformly on $\cS \times \cA$, and invoking Lemma~\ref{lem:fkl:mab:perf-diff} context-wise gives
\begin{align*}
    \suboptfkl(\hat{\pi}) &\leq 2 \eta \EE_{s \sim \rho} \Bigg[ \sum_{a \in \cA } \bigg( \frac{\pistarfkl^2(a|s)}{\piref(a|s)}(\fgt(s,a) - \fps(s,a))^2 \bigg) \Bigg] \\
    & = 2 \eta \EE_{s \sim \rho} \Bigg[ \sum_{a \in \cA } \bigg[ \bigg(\frac{\pistarfkl(a|s)}{\piref(a|s)}\bigg)^2\piref(a|s)(\fgt(s,a) - \fps(s,a))^2 \bigg] \Bigg] \\
    & \leq 2 \eta (C^{\pi^*})^2\EE_{(s,a) \sim \rho\times \piref} \big[\big(\fgt(s,a) - \fps(s,a)\big)^2 \big] \\
    &\leq 2 \eta (C^{\pi^*})^2  \sum_{s \in \cS} \rho(s) \sum_{a \in \cA} \piref(a|s)\frac{128\log^2(2|\cS||\cA|/\delta)}{n\rho(s)\piref(a|s)} \\
    & = \tilde{O} \big(\eta (C^{\pi^*})^2 |\cS| |\cA| n^{-1} \big),
\end{align*}
where the second inequality holds by $\pistarfkl(a|s)/\piref(a|s) \leq C^{\pi^*}$. By Lemma~\ref{lem:pessimistic-tabular}, $\confbandittb(\delta) \cap \confbanditbin(\delta)$ holds with probability at least $1-\delta$, which finishes the proof.
\end{proof}

\color{black}



\section{Missing Proofs in \Cref{sec:lb}}\label{app:lb-proofs}
We first present (for each given context) a suboptimality decomposition lemma applicable for general learners.\footnote{We drop the dependency on context in \Cref{lem:fkl:learner-agnostic-perf-diff} as the suboptimality gaps on each $s \in \cS$ are purely additive.}

\begin{lemma}\label{lem:fkl:learner-agnostic-perf-diff}
    Let $\hat{\pi} \in \Delta(\cA)$ be \emph{any} distribution with $\supp(\hat\pi) = \cA$, then under \Cref{assump:full-supp},
\begin{align}
    \suboptfkl(\hat\pi) = \eta^{-1} \EE_{a \sim \piref}\bigg[ \psi\Big( \frac{\hat{\pi}(a)} {\pistarfkl(a)} \Big) \bigg],
\end{align}
where $\psi(x) \coloneqq x - 1 - \log x$.
\end{lemma}
\begin{proof}[Proof of \Cref{lem:fkl:learner-agnostic-perf-diff}]
\Cref{lem:fkl:opt} implies that $\forall a \in \cA$,
\begin{align*}
    r(a) = \lambda - \eta^{-1} \frac{\piref(a)}{\pistar(a)},
\end{align*}
which, substituted into $\subopt(\cdot)$, yields
\begin{align*}
    \mathrm{LHS} &\coloneqq \obj(\pistar;r) - \obj(\hat{\pi}; r) \\
    &= \dotp{r}{\pistar - \hat\pi} + \eta^{-1} \sum_{a \in \cA} \piref(a) \log\frac{\piref(a)}{\hat{\pi}(a)} - \eta^{-1}\sum_{a\in\cA} \piref(a) \log \frac{\piref(a)}{\pistar(a)} \\
    &= \dotp{\lambda \one - \eta^{-1} \frac{\piref}{\pistar}}{\pistar - \hat\pi} - \eta^{-1} \sum_{a \in \cA} \piref(a) \log \frac{\hat{\pi}(a)}{\pistar(a)} \\
    &= -\eta^{-1} \sum_{a \in \cA} \piref(a) \Big( 1 - \frac{\hat\pi(a)}{\pistar(a)} \Big) - \eta^{-1} \sum_{a \in \cA} \piref(a) \log \frac{\hat{\pi}(a)}{\pistar(a)} \\
    &= \eta^{-1} \EE_{a \sim \piref} \bigg[ \frac{\hat\pi(a)}{\pistar(a)} - 1 - \log \frac{\hat\pi(a)}{\pistar(a)} \bigg],
\end{align*}
where the penultimate equality is because both $\pistar$ and $\hat{\pi}$ are probability distributions.
\end{proof}

\begin{remark}
   Although the \emph{pattern} of \Cref{lem:fkl:learner-agnostic-perf-diff} fits into \Cref{eq:lem:fkl:value-based-perf-diff} exactly, their applicable scopes are different because the first equality of \Cref{eq:lem:fkl:value-based-perf-diff} is a specialization of \Cref{lem:simplex-conjugate-bregman} to forward KL, which only needs mild regularity conditions of the regularizer $h$ but requires $\hat{\pi}$ to be the maximizer of $\objfkl(\cdot; \hat{r})$. On the other hand, \Cref{lem:fkl:learner-agnostic-perf-diff} is tailored to forward KL but $\hat{\pi}$ here can be essentially an arbitrary probability distribution.
\end{remark}

\subsection{The Fast Rate Lower Bound}

\begin{proof}[Proof of \Cref{thm:fkl:mab:lb-ref=unif-A-arms:fast-rate:contextual}]
Let $K \coloneqq A/2$. We consider MABs with context set $\cS = \seq{S}$ and action set $\cA = \overline{[A]}$. We fix $\rho = \mathsf{Uniform}(\cS)$ $\piref = \mathsf{Uniform}(\cA)$, where every reward follows a Bernoulli distribution. We consider the subset of instances with the following parameterization. Let $\cV = \{\pm 1\}^{S\times K}$ and we use $v_{s,k}$ to denote its $(s,k)$-th entry given any $v \in \cV$. The reward distribution given $v \in \cV$ is given as follows:
\begin{align*}
    r_v(s, 2k) = 0.5 - v_k c,\ r_v(s, 2k+1) = 0.5 + v_k c; \quad \forall \ (s,k) \in \seq{S} \times \seq{K},
\end{align*}
where $c \in (0, 0.25)$ is to be specified.

We now compute the optimal policy $\pi_v$ for each instance parameterized by $v \in \cV$. Fix any $s$, we know by \Cref{lem:fkl:opt} that $\pistar_v(\cdot|s) = (\eta A)^{-1}/(\lambda - r_v(s, \cdot))$; and thus all instances share the same $\lambda$ over all the contexts $s \in \cS$ and $\lambda$ is determined by 
\begin{align*}
    \sum_{k=0}^{K-1} \Big( \frac{(\eta A)^{-1}}{\lambda - 0.5-c} + \frac{(\eta A)^{-1}}{\lambda - 0.5 + c} \Big) = 1 \Longrightarrow \lambda - 0.5 = \frac{1 + \sqrt{1 + 4\eta^2c^2}}{2\eta},
\end{align*}
which is because $\lambda > \max_{v,a} r_v(a)$. Let $\iota = \lambda - 0.5$, then $\iota = \eta(\iota^2 - c^2)$, and thus we set $\kappa_{v,s,a} \coloneqq (-1)^{a - 2\floor{a/2}} \cdot v_{s,\floor{a/2}}$ and $\omega \coloneqq c / \iota = 2\eta c / \big( 1 + \sqrt{1 + 4\eta^2c^2}\big)$ to obtain
\begin{align*}
    \pistar_v(a|s) &= \frac{1}{\eta A} \cdot \frac{1}{\iota + (-1)^{a - 2\floor{a/2}} \cdot v_{s, \floor{a/2}} \cdot c} \\
    & = \frac{\iota - \kappa_{v,s,a} c}{\eta A (\iota^2 - c^2)} \\
    & = A^{-1}(1 - \kappa_{v,s,a} c/\iota) \\
    & = \frac{1 - \kappa_{v,s,a} \omega}{A}.
\end{align*}
Note that \Cref{lem:fkl:learner-agnostic-perf-diff} implies
\begin{align*}
    \subopt(\hat{\pi}; r_v) &=  \sum_{s=0}^{S-1}\sum_{k=0}^{K-1} \frac{1}{\eta SA}\bigg[ \psi\Big( \frac{\hat{\pi}(2k |s)}{\pistar_v(2k|s)} \Big) + \psi\Big( \frac{\hat{\pi}(2k+1|s)}{\pistar_v(2k+1|s)} \Big) \bigg] \\
    &= \sum_{s=0}^{S-1}\sum_{k=0}^{K-1}\underbrace{ \frac{1}{\eta SA} \bigg[ \psi\Big( \frac{A\hat{\pi}(2k|s)}{1 - v_k \omega} \Big) + \psi\Big( \frac{A\hat{\pi}(2k+1|s)}{1 + v_k \omega} \Big) \bigg] }_{\eqqcolon \Psi_{s,k}(\hat{\pi};v)}.
\end{align*}
Then for $v \sim_{s,j} v'$ (i.e., $v, v' \in \{\pm 1\}^{S \times K}$ differ only in the $(s,j)$-th coordinate),
\begin{align}
     & \Psi_{s,j}(\hat{\pi};v) +  \Psi_{s,j}(\hat{\pi};v') \notag\\
     & \quad = \frac{1}{\eta SA}\bigg( \psi\Big( \frac{A \hat\pi(2j|s)}{1 - \omega}\Big) + \psi\Big( \frac{A \hat\pi(2j|s)}{1 + \omega}\Big) +  \psi\Big( \frac{A \hat\pi(2j+1|s)}{1 - \omega}\Big) + \psi\Big( \frac{A \hat\pi(2j+1|s)}{1 + \omega}\Big) \bigg). \label{eq:fkl:fast-rate-A-arms:separation-1}
\end{align}
Now we specify $c = \sqrt{SA/n}$, which is less than $0.25$ by design. Recall that $\psi(x) = x -1-\log x$, elementary calculus shows that $o \mapsto \psi\big( o/(1+\omega) \big) + \psi\big(o/ (1-\omega) \big)$ is minimized at $o = 1 - \omega^2$, and thus
\begin{align}
    \Cref{eq:fkl:fast-rate-A-arms:separation-1} &\geq  \frac{2}{\eta SA} \log \frac{1}{1 - \omega^2} \notag\\
    &= \frac{2}{\eta SA} \cdot \log \frac{1 + \sqrt{1 + 4\eta^2 c^2}}{2} \notag\\
    & \geq  \frac{2}{\eta SA} \cdot \frac{\sqrt{1 + 4\eta^2c^2} - 1}{4} \notag \\
    & = \frac{\sqrt{1 + 4\eta^2c^2} - 1}{2\eta SA} \notag\\
    & \geq \frac{2\eta c^2}{SA(1+\sqrt{17})} \notag \\
    & = \Omega(\eta n^{-1}), \label{eq:fkl:fast-rate-cb:separation-2}
\end{align}
where the second inequality is because $\eta^{-1} > c/2$ and $\log(1+x) \geq x/2$ for $x\in[0,2]$, and the last inequality is again due to $\eta^{-1} > c/2$.

Now, given any reward function $r$, let $\PP_{\piref, \mathtt{r}}$ be the distribution of $(a, \mathtt{r})$ for $(s,a) \sim \rho \times \piref$ and $\mathtt{r} \sim \mathsf{Bernoulli}(r(a))$. We further use $\PP_{\piref, v}$ for $v\in \cV$ to denote $\PP_{\piref, r}$ if $r$ is parameterized by $v\in \cV$. The divergence decomposition lemma~\citep[Lemma~15.1]{lattimore2020bandit} then implies that for any $(s, j) \in \seq{S} \times \seq{K}$ and $v \sim_{s,j} v'$,
\begin{align}
    \kl{\PP_{\piref, v}^{\otimes n}}{\PP_{\piref, v'}^{\otimes n}} &= \frac{n}{SA} \kl{\mathrm{Bern}(r_v(2j)}{\mathrm{Bern}(r_{v'}(2j)} \\
    & \quad + \frac{n}{SA} \kl{\mathrm{Bern}(r_v(2j+1)}{\mathrm{Bern}(r_{v'}(2j+1)} \notag\\
    &\leq  \frac{128}{3}\frac{nc^2}{SA} = \frac{128}{3}, \label{eq:fkl:fast-rate-cb:kl}
\end{align}
where we utilize $\kl{\mathrm{Bern}(p)}{\mathrm{Bern}(q)} \leq (p-q)^2/[q(1-q)]$ and $c \in [0, 0.25]$. By \Cref{lem:assouad}, \Cref{eq:fkl:fast-rate-cb:separation-2,eq:fkl:fast-rate-cb:kl} together imply the desired lower bound.
\end{proof}

\subsection{Lower Bound for Two-Armed Bandits: Phase Transition}\label{subsec:fkl:two-arm}

\begin{theorem}\label{thm:fkl:mab:lb-ref=unif}
Let $\hat\pi \in \Delta(\cA)$ be any estimator from $n$ action-reward pairs, then for any $n > 16$ and $\eta > 0$,
\begin{align}
\sup_{(\cA, r, \piref) \in \mathrm{MAB}_2} \EE_{\cD \sim P_{\piref, r}} \suboptfkl(\hat\pi; \cA, r, \piref) \gtrsim  \eta n^{-1} \wedge n^{-1/2}  . \label{eq:fkl:lb}
\end{align}
\end{theorem}
\begin{proof}
    We construct two MABs sharing $\cA = \{0, 1\}$ and $\piref = \mathsf{Uniform}(\cA)$,\footnote{In this proof, we drop notations that are clear in the context to make the presentation less heavy.} where the rewards follow Bernoulli distribution with means specified as
\begin{align}
    \acute{r} = (0.5+c, 0.5 - c), \grave{r} = (0.5 - c, 0.5 + c).
\end{align}
Here, $c \in (0, 0.25)$ will be specified later.
By \Cref{lem:fkl:opt}, $\pistar(\cdot) =  ({2\eta})^{-1} \big({\lambda - r(\cdot)}\big)^{-1}$,
which means the two instances $\acute{r}$ and $\grave{r}$ share the same $\lambda$, i.e., the $\lambda$ in both cases satisfies
\begin{align}
    (\lambda - 0.5  - c)(\lambda - 0.5 + c) = \frac{2\lambda - 1}{2\eta}. \label{eq:fkl:lambda}
\end{align}
To get an eligible distribution $\pistar$, we require $\lambda >  0.5 + c$ according to \Cref{eq:fkl:sol}, which means \Cref{eq:fkl:lambda} yields
\begin{align}
    2\lambda - 1 = \eta^{-1} + \sqrt{\eta^{-2} + 4c^2}. \label{eq:fkl:2lam-1}
\end{align}
Thus, for any estimator $\hat\pi \in \Delta(\{0, 1\})$,
\begin{align}
    \subopt(\hat\pi; \acute{r}) + \subopt(\hat\pi; \grave{r}) &= \dotp{\acute{r}}{\acute{\pistar}} +   \dotp{\grave{r}}{\grave{\pistar}} - \overbrace{\dotp{\acute{r} + \grave{r}}{\hat\pi}}^{=1} + \overbrace{2\eta^{-1}\kl{\piref}{\hat\pi}}^{\geq0} \notag\\
    &\quad\ -\eta^{-1}\kl{\piref}{\acute{\pistar}} - \eta^{-1}\kl{\piref}{\grave{\pistar}} \notag\\
    &\geq 2\dotp{\acute{r}}{\acute{\pistar}}- 1 - 2\eta^{-1}\kl{\piref}{\acute{\pistar}} \notag\\
    &= 2\Big( {(0.5 + c)}\acute{\pistar}(0) + {(0.5 - c)}\acute{\pistar}(1)  \Big) - 1 \notag\\
    &\quad\ -\eta^{-1}\Big( \log\eta(\lambda - 0.5 - c) + \log\eta(\lambda - 0.5 + c) \Big) \notag\\
    &= \frac{c\eta^{-1}}{\eta(\lambda - 0.5 - c)} - \frac{c\eta^{-1}}{\eta(\lambda - 0.5 + c)}  \notag\\
    &\quad\ - \eta^{-1}\Big( \log\eta(\lambda - 0.5 - c) + \log\eta(\lambda - 0.5 + c) \Big) \notag\\
    &= \frac{4c^2}{2\lambda - 1} - \eta^{-1}\log\frac{\eta(2\lambda - 1)}{2}  \label{eq:fkl:pre-lb-shared} \\
    &= \frac{4c^2}{\eta^{-1} + \sqrt{\eta^{-2} + 4c^2}} - \eta^{-1}\log\Big( 1 + \frac{\sqrt{1+4c^2\eta^2} - 1}{2} \Big). \label{eq:fkl:lb-shared}
\end{align}
where \Cref{eq:fkl:pre-lb-shared} is due to \Cref{eq:fkl:lambda} and \Cref{eq:fkl:lb-shared} follows from \Cref{eq:fkl:2lam-1}.
Given any mean function $r$, let $P_{\piref, r}$ be the distribution of $(a, \mathtt{r})$ for $a \sim \piref$ and $\mathtt{r} \sim \mathsf{Bernoulli}(r(a))$. The tensorization property of $\mathsf{KL}$ \citep[Chapter~2]{polyanskiy2025information} then implies
\begin{align}
    \kl{P_{\piref, \acute r}^{\otimes n}}{P_{\piref, \grave r}^{\otimes n}} = n \kl{P_{\piref, \acute r} }{P_{\piref, \grave r} } = 2nc \log\frac{1+2c}{1-2c} \leq 16nc^2, \label{eq:fkl:calc}
\end{align}
where the inequality follows from the requirement $c \in (0, 0.25)$. Since $n > 16$. we can set $c = n^{-1/2}$.

\noindent\textbf{Case $\eta^{-1} \leq c  / 2 = 0.5n^{-1/2}$.} Using standard $\sqrt{x+y} \leq \sqrt{x} + \sqrt{y}$ and $\log(1+x) \leq x$ for $x,y\geq0$,
\begin{align*}
    \Cref{eq:fkl:lb-shared} &\geq \frac{4c^2}{2\eta^{-1} + 2c} - \eta^{-1}\log(1 + c\eta) \geq \frac{2c^2}{\eta^{-1} + c} - c \geq {c}/{3},
\end{align*}
which, combined with \Cref{lem:two-point,eq:fkl:calc}, yields
\begin{align*}
    \text{LHS of \Cref{eq:fkl:lb}} \gtrsim n^{-1/2}.
\end{align*}
\noindent\textbf{Case $\eta^{-1} > c  / 2 = 0.5n^{-1/2}$.} Using $1 =\sup_{x > 0} x^{-1}\log(1+x)$,
\begin{align*}
    \Cref{eq:fkl:lb-shared} &= \frac{2}{1 + \sqrt{1 + 4c^2\eta^2}}\bigg(2\eta c^2 - \eta c^2 \cdot \frac{\log\Big( 1 + \frac{\sqrt{1+4c^2\eta^2} - 1}{2} \Big)}{\frac{\sqrt{1+4c^2\eta^2} - 1}{2}} \bigg) \notag \\
    &\geq \frac{2}{1 + \sqrt{17}} (2\eta c^2 - \eta c^2) \geq \eta c^2/3,
\end{align*}
which, combined with \Cref{lem:two-point,eq:fkl:calc}, yields
\begin{align*}
    \text{LHS of \Cref{eq:fkl:lb}} \gtrsim \eta n^{-1}.
\end{align*}
\end{proof}

\section{Other Missing Proofs}\label{app:other-missing}

\subsection{Missing Proofs of Supporting Lemmas}\label{app:subsec:supporting-proofs}

\begin{proof}[Proof of \Cref{lem:fkl:opt}]
Under \Cref{assump:full-supp}, consider the primal problem\footnote{We drop the context $s \in \cS$ and its conditioning in this proof since the optimal solution for each context is fully decoupled.}
\begin{equation}
    \begin{aligned}
            \min_{\pi} &- \dotp{r}{\pi} + \eta^{-1}\kl{\piref}{\pi} \\
    \text{s.t. } &\dotp{\one}{\pi} = 1,   -\pi \leq \zero;
    \end{aligned}\tag{P}\label{eq:fkl:primal}
\end{equation}
whose Lagrangian with dual variables $(\nu, \lambda)$ is
\begin{align*}
    \cL(\pi, \nu, \lambda) \coloneqq - \dotp{r}{\pi} + \eta^{-1}\kl{\piref}{\pi} + \lambda (\dotp{\one}{\pi} - 1) - \dotp{\nu}{\pi}.
\end{align*}
Since $\piref(a) > 0$ for all $a \in \cA$, any feasible $\pi$ with finite objective must satisfy $\pi(a) > 0$ for all $a \in \cA$. Moreover, since the constraints in \Cref{eq:fkl:primal} are affine, Slater's condition holds, implying strong duality. The KKT conditions at an optimal primal-dual tuple $(\pi, \nu, \lambda)$ read
\begin{align}
    \nabla_{\pi} \cL(\pi, \nu, \lambda) &= \zero, \label{eq:fkl:1st-cond} \\
    \pi &\geq \zero, \dotp{\one}{\pi} = 1, \label{eq:fkl:primal-feas} \\
    \nu &\geq \zero, \pi \odot\nu = \zero. \label{eq:fkl:dual-feas}
\end{align}
Because $\pi(a) > 0$ for all $a \in \cA$, complementary slackness in \Cref{eq:fkl:dual-feas} forces $\nu = \zero$. \Cref{eq:fkl:1st-cond} then yields
\begin{align*}
    \forall a \in \cA, \quad \pi(a) = \frac{\eta^{-1}\piref(a)}{\lambda - r(a)},
\end{align*}
which requires $\lambda > \max_{a \in \cA} r(a)$ and that $\lambda$ solves
\begin{align}
    \sum_{a \in \cA}\frac{\eta^{-1}\piref(a)}{\lambda - r(a)} = 1. \label{eq:fkl:normalization}
\end{align}
Since $\lambda \mapsto \sum_{a\in\cA}{\eta^{-1}\piref(a)}/\big({\lambda - r(a)}\big)$ is strictly decreasing and continuous on $(\max_{a \in \cA} r(a), +\infty)$ with range $(0, +\infty)$, the root $\lambda$ of \Cref{eq:fkl:normalization} exists and is unique. By \Cref{lem:fkl:sc}, \Cref{eq:fkl:primal} is strictly convex on $\ri(\Delta(\cA))$, so $\pistarfkl$ is the unique global solution.
\end{proof}

\begin{proof}[Proof of \Cref{lem:fkl:sc}]
    Pick an arbitrary $\pi \in \ri\big(\Delta(\cA)\big)$, let $F(\cdot) = \kl{\piref}{\cdot}$. Then
\begin{align}
    \nabla F(\pi) = -\frac{\piref}{\pi} \Rightarrow \nabla^2 F(\pi) = \diag\Big\{\frac{\piref}{\pi^2}\Big\}. \label{eq:hessian-forward-kl}
\end{align}
Here, both divisions are element-wise. $\forall \mu, \nu \in \Delta(\cA)$,
\begin{align*}
    (\mu - \nu)^\top  \nabla^2 F(\pi)   (\mu - \nu)  &= \sum_{a\in\cA} \frac{\piref(a)\big(\mu(a) - \nu(a)\big)^2}{\big(\pi(a)\big)^2} \\
    &\geq \alpha \sum_{a\in\cA} \frac{\big(\mu(a) - \nu(a)\big)^2}{\big(\pi(a)\big)^2}  \\
    & \ge \alpha  \frac{ \Big(\sum_{a\in\cA}{|\mu(a) - \nu(a)|}\Big)^2}{ \sum_{a\in\cA}\big(\pi(a)\big)^2} \\
    & \geq 4\alpha \cdot \big(\tv{\mu}{\nu}\big)^2,
\end{align*}
where the second inequality follows from Cauchy–Schwarz.\footnote{
Technically, $\dim \Delta(\cA) = A-1 < A$; and thus by $\nabla F(\pi)$ (resp. $\nabla^2 F(\pi)$) for $\pi \in \ri(\Delta(\cA))$ in this proof, we actually mean the Euclidean gradient (resp. Hessian) of the analytical continuation of $F$ defined as
$F_\mathrm{ext}(\mu) \coloneqq \dotp{\piref}{\log \frac{\piref}{\mu}}$ for $\mu \in \RR_{++}^{\cA}$ and $+\infty$ otherwise. A similar convention follows in the proof of \Cref{lem:fkl:opt}.
} Here, we recall that $\mathsf{TV}$ is equivalent to the $L^1$ distance up to a multiplicative constant $2$, and thus the conclusion follows.
\end{proof}

\begin{proof}[Proof of \Cref{lem:fkl:mab:perf-diff-fail}]
For $u \in [0, 1]$, $r_u \coloneqq (1-u)r + u\hat{r}$, and $\pi_u = \argmax_{\pi \in \Delta(\cA)} \objfkl(\pi; r_{u})$, the latter of which exists and is unique by \Cref{lem:fkl:opt}. Let
\begin{align*}
    V(u) \coloneqq \objfkl(\pi_u; r_u), \quad \Phi(u) \coloneqq \objfkl(\pi_u; r).
\end{align*}
\Cref{lem:fkl:opt} and the definition of $\objfkl$ guarantee that both $V(\cdot)$ and $u \mapsto \objfkl(\pi, r_u)$ is differentiable given any $\pi \in \Delta(\cA)$. Therefore, by \Cref{lem:envelope}, $\forall u \in (0, 1)$,
\begin{align*}
    V'(u) = \Big\langle \frac{\ud}{\ud u}r_u, \pi_u \Big\rangle = \dotp{\hat{r} - r}{ \pi_u},
\end{align*}
which, substituted to $\Phi(u) = V(u) + \dotp{r - r_u}{\pi_u} = V(u) + u\dotp{\hat{r}-  r}{\pi_u}$, yields
\begin{align}
  \forall u \in (0,1),  \Phi'(u) &= V'(u) + \frac{\ud}{\ud u}\Big( u\dotp{r- \hat{r}}{\pi_u} \Big) \notag\\
  &= \dotp{\hat{r} - r}{ \pi_u} + \dotp{r- \hat{r}}{\pi_u} + u \Big\langle{r- \hat{r}}, {\frac{\ud}{\ud u}\pi_u} \Big\rangle = u \Big\langle {r- \hat{r}}, {\frac{\ud}{\ud u}\pi_u} \Big\rangle.
\end{align}
By definition,
\begin{align}
    \mathrm{LHS\ of\ \Cref{eq:lem-perf-diff-full-supp}} &= \Phi(0) - \Phi(1) = - \int_{0}^{1} \Phi'(u) \ud u = \int_{0}^{1} u \Big\langle {\hat{r}- {r}}, {\frac{\ud}{\ud u}\pi_u} \Big\rangle \ud u. \label{eq:perf-diff-full-supp-integral-rep}
\end{align}
Let $h(\cdot) \coloneqq \eta^{-1}\kl{\piref}{\cdot}$. Since $h$ is strictly convex and continuously differentiable on $\ri(\Delta(\cA))$, $\mathrm{dom}(h) \coloneqq \{ \pi \in \Delta(\cA): h(\pi) < \infty\} = \ri(\Delta(\cA))$, and the uniqueness of $\pi_u$ is guaranteed by \Cref{lem:fkl:opt}, applying \Cref{lem:simplex-conjugate-duality} to $h$ yields that $\forall u \in (0,1)$, $\pi_u = \nabla h^*(r_u)$, where the notation $h^*$ is consistent with \Cref{lem:simplex-conjugate-duality}. Consequently,
\begin{align}
    {\frac{\ud}{\ud u}\pi_u} =  {\frac{\ud}{\ud u}} \nabla h^*(r_u) = \big[ \nabla^2 h^*(r_u) \big] \cdot \frac{\ud}{\ud u}r_u = \big[ \nabla^2 h^*(r_u) \big](\hat{r} - r), \label{eq:pi-u-derivative}
\end{align}
where $\hat{r} - r$ is again viewed as a vector. Substituting \Cref{eq:pi-u-derivative} into \Cref{eq:perf-diff-full-supp-integral-rep} yields
\begin{align*}
    \mathrm{LHS\ of\ \Cref{eq:lem-perf-diff-full-supp}} &= \int_{0}^{1} u (\hat{r} - r)^\top \big[ \nabla^2 h^*(r_u) \big](\hat{r} - r) \ud u \\
    &= \int_{0}^{1} u (r - \hat{r})^\top \big[ \nabla^2 h(\pi_u) \big]^{-1} (r - \hat{r}) \ud u - \int_{0}^{1} u \frac{ \big( \one^\top \big[ \nabla^2 h(\pi_u) \big]^{-1} (r - \hat{r}) \big)^2 }{ \one^\top \big[ \nabla^2 h(\pi_u) \big]^{-1} \one } \ud u \\
    &= \eta \int_{0}^{1} \bigg( \sum_{a} \Big( \frac{\pi_u^2}{\piref}(r - \hat{r})^2 \Big)(a) - \frac{ \Big( \sum_{a} \big(\frac{\pi_u^2}{\piref}(r-\hat{r})\big)(a) \Big)^2 }{\sum_{a} \frac{\pi_u^2}{\piref}(a)} \bigg)u\ud u,
\end{align*}
where $\frac{\pi_u^2}{\piref}(r - \hat{r})^2$ is viewed as a function, the second equality follows from \Cref{lem:simplex-conjugate-hessian} and \Cref{assump:full-supp} and the last equality follows from direct calculations detailed in \Cref{eq:hessian-forward-kl}. The final result thus follows from the \emph{second} mean value theorem for definite integrals.


 
\end{proof}

\begin{proof}[Proof of \Cref{eq:lem:fkl:mab:perf-diff-fail--derivative}]

Recall that $F(\bar{u})$ denotes \Cref{eq:lem-perf-diff-full-supp}, $\tilde{F}(\bar{u})$ denotes \Cref{eq:lem-perf-diff-full-supp-coarse};
$w_u(a) \coloneqq \frac{\pi_u(a)}{\piref(a)}$, $Z_u \coloneqq \sum_{a} \frac{\pi_u^2(a)}{\piref(a)}$, and $\mu_u(a) \coloneqq \frac{\pi_u^2(a)}{\piref(a)}/ Z_u $, and $g \coloneqq r - \hat{r}$.
First, the chain rule directly gives
\begin{align}
    \tilde{F}'(u) = \eta \inner{ g \odot w_u }{ g \odot \frac{\ud}{\ud u} \pi_u }. \label{eq:derivative-2st-part-chain-rule}
\end{align}
We follow the same pipeline as that in the proof of \Cref{lem:fkl:mab:perf-diff-fail} to calculate $\frac{\ud}{\ud u} \pi_u$. Specifically, we again substitute $\big[\nabla^2 h(\pi_u) \big]^{-1} = \eta \cdot \diag(\pi_u^2 / \piref)$ into \Cref{lem:simplex-conjugate-hessian}, and then substitute the result into \Cref{eq:pi-u-derivative} to obtain
\begin{align}
    \eta^{-1} \frac{\ud}{\ud u} \pi_u = Z_u \big(- \mu_u \odot g + \inner{\mu_u}{ g} \mu_u \big). \label{eq:derivative-wrt-u}
\end{align}

We set $X(a) \coloneqq-  w_u(a) \cdot g(a)$ and substitute \Cref{eq:derivative-wrt-u} into \Cref{eq:derivative-2st-part-chain-rule} to obtain
\begin{align}
    \eta^{-2} \tilde{F}'(u) &=  \EE_{a \sim \piref} X(a)^3 -  \Big( \EE_{a \sim \pi_u} X(a)^2 \Big)  \EE_{a \sim \pi_u} X(a) \big/ Z_u. \label{eq:derivative-2st-part-post-substitution}
\end{align}

Another application of the chain rule yields that $\tilde{F}'(u) - F'(u)$
\begin{align}
     &= \frac{\ud}{\ud u} \bigg(\frac{\eta Z_u}{2} \cdot    \inner{\mu_u(a)}{g(a)}^2   \bigg)  \notag\\
    &= - \frac{2\eta \EE_{a\sim \pi_u} X(a)}{Z_u} \inner{w_u}{g \odot \frac{\ud}{\ud u} \pi_u} -  \frac{\eta}{Z_u^2} \Big( \EE_{a \sim \pi_u} X(a) \Big)^2 \inner{w_u}{ \frac{\ud}{\ud u} \pi_u }. \label{eq:derivative-2nd-part-pre-substitution}
\end{align}

Substituting \Cref{eq:derivative-wrt-u} into \Cref{eq:derivative-2nd-part-pre-substitution} yields
\begin{align*}
    \eta^{-2} \cdot \Cref{eq:derivative-2nd-part-pre-substitution} &= \frac{\EE_{a\sim \pi_u} X(a) }{Z_u} \cdot \Big( \EE_{a \sim \pi_u} X(a)^2 - \EE_{a \sim \mu_u} X(a) \EE_{a \sim \pi_u} X(a) \Big) \\
    & - \Big( \EE_{a \sim \pi_u} X(a) / Z_u \Big)^2 \cdot \Big( Z_u \cdot \EE_{a \sim \mu_u} X(a) - \inner{w_u}{\mu_u} \cdot \EE_{a \sim \pi_u} X(a) \Big),
\end{align*}
which, added to \Cref{eq:derivative-2st-part-post-substitution}, yields
\begin{align}
    \eta^{-2} F'(u) &= \EE_{a \sim \piref} X(a)^3 - \frac{3}{Z_u} \EE_{a \sim \pi_u}X(a)^2\EE_{a \sim \pi_u} X(a) \notag\\
    &+ \frac{3}{Z_u}\Big(\EE_{a \sim \pi_u} X(a) \Big)^2 \EE_{a \sim \mu_u} X(a) - \bigg(\frac{\EE_{a \sim \pi_u} X(a)}{Z_u} \bigg)^3 \cdot \inner{w_u}{\mu_u} Z_u. \label{eq:derivative-pre-final}
\end{align}
$\eta^2 \cdot \Cref{eq:derivative-pre-final}$ is exactly $- \eta^2Z_u \cdot \EE_{a \sim \mu_u} \Big[ w_u(a) \cdot \big( g(a) - \EE_{a \sim \mu_u} g(a) \big)^3 \Big]$.

\end{proof}

\subsection{Discussion on \Cref{assump:full-supp} and Non-Full-Support Reference Policies}\label{app:subsec:full-supp}

We elaborate on how the forward-KL-regularized offline bandit problem behaves when $\supp(\piref(\cdot|s)) \subsetneq \cA$ for some context $s \in \cS$. Unlike the reverse KL divergence $\kl{\pi(\cdot|s)}{\piref(\cdot|s)}$---which evaluates to $+\infty$ whenever $\supp(\pi(\cdot|s)) \not\subseteq \supp(\piref(\cdot|s))$ and thus automatically confines any finite-value policy to $\cA_s^{\mathrm{eff}} \coloneqq \supp(\piref(\cdot|s))$---the forward KL divergence $\kl{\piref(\cdot|s)}{\pi(\cdot|s)} = \sum_{a \in \cA_s^{\mathrm{eff}}} \piref(a|s)\log\frac{\piref(a|s)}{\pi(a|s)}$ imposes only a finite normalization penalty on unsupported mass: writing $\pi(\cdot|s) = (1-q)\mu + q\nu$ with $\mu \in \Delta(\cA_s^{\mathrm{eff}})$, $\nu \in \Delta(\cA \setminus \cA_s^{\mathrm{eff}})$, and $q \in [0, 1)$, one has $\kl{\piref(\cdot|s)}{\pi(\cdot|s)} = \kl{\piref(\cdot|s)}{\mu} - \log(1-q)$.

For rewards in $[0, 1]$ and strong regularization $\eta \leq 1$, this normalization cost dominates any reward gain ($\objfkl(\pi; r) - \objfkl(\mu; r) \leq q + \eta^{-1}\log(1-q) < 0$ for all $q \in (0, 1)$), so $q^* = 0$ still holds. Whenever $\eta > 1$, however, unrestricted optimization becomes unidentifiable from offline data. Concretely, consider a single-context two-action instance with $\cA = \{a_1, a_2\}$, $\piref = (1, 0)$, and mean reward $r = (0, 1)$: parameterizing $\pi \in \Delta(\cA)$ by $q \coloneqq \pi(a_2) \in [0, 1)$ gives $\objfkl(\pi; r) = q + \eta^{-1}\log(1 - q)$, whose unique maximizer for $\eta > 1$ is $q^* = 1 - \eta^{-1} > 0$, placing strictly positive mass on $a_2 \notin \supp(\piref)$. Meanwhile, under $\piref = (1, 0)$, the dataset $\cD$ consists exclusively of pulls of $a_1$, so $\PP_{\piref, r}^{\otimes n}$ is identical under $r = (0, 1)$ ($q^* = 1 - \eta^{-1}$) and under $r' = (0, 0)$ ($q^* = 0$). Thus, without cross-action structural assumptions, unrestricted optimization over $\Delta(\cA|\cS)$ is unidentifiable in general when $\piref$ lacks full support.

Conversely, if one formulates the offline policy optimization problem over the support-restricted policy class
\begin{align*}
    \Pi_{\mathrm{eff}} \coloneqq \big\{ \pi \in \Delta(\cA|\cS) : \supp(\pi(\cdot|s)) \subseteq \cA_s^{\mathrm{eff}}, \ \forall s \in \cS \big\},
\end{align*}
then for each context $s \in \cS$ the optimization reduces to the admissible action set $\cA_s^{\mathrm{eff}}$, on which $\piref(\cdot|s)$ satisfies \Cref{assump:full-supp} by definition. All upper bounds in \Cref{sec:ub} apply directly to $\max_{\pi \in \Pi_{\mathrm{eff}}} \objfkl(\pi)$ upon replacing $\cA$ by $\cA_s^{\mathrm{eff}}$, with no quantitative dependence on $\min_{s \in \cS, a \in \cA_s^{\mathrm{eff}}} \piref(a|s)$.

\subsection{Unknown Context Distribution and Empirical Disagreement Bonus}\label{app:subsec:empirical-bonus}

In \Cref{sec:ub-gfa}, \Cref{algorithm:bandit-pess-gfa} uses the population disagreement width $D_{\fcl}((s,a); \piref)$, whose denominator takes expectation over $\nu \coloneqq \rho \times \piref$. Here we show that when $\rho$ is unknown, a sample-split regularized empirical disagreement width~\citep{zhu2023principled,di2023pessimistic} preserves the statistical rate of \Cref{thm:function-approximation} up to universal constants and logarithmic factors.

\noindent\textbf{Sample-split construction and oracle model.} Assume $n \geq 2$, write $X_i \coloneqq (s_i, a_i)$ for the $i$-th sample in $\cD$ and $x \coloneqq (s, a) \in \cS \times \cA$, and define $\|f\|_{\nu}^2 \coloneqq \EE_{X \sim \nu}[f(X)^2]$ so that $D_{\pistar}^2 = \EE_{X \sim \rho \times \pistarfkl}[D_{\fcl}^2(X; \piref)]$. After a uniform random permutation of indices, split $\cD$ into $I_{\mathrm{geo}} \coloneqq \{1, \dots, m_0\}$ and $I_{\mathrm{reg}} \coloneqq \{m_0 + 1, \dots, n\}$, where $m_0 \coloneqq \floor{n/2}$ and $m_1 \coloneqq n - m_0 = \ceil{n/2}$ are independent halves with $m_1 \geq m_0 \geq n/3$. Define the empirical semi-norm on $I_{\mathrm{geo}}$ by $\|f\|_{\hat{\nu}_0}^2 \coloneqq m_0^{-1} \sum_{i \in I_{\mathrm{geo}}} f(X_i)^2$, and compute the least-squares estimator on $I_{\mathrm{reg}}$ via
\begin{align}
    \fls \in \argmin_{g \in \fcl} \sum_{i \in I_{\mathrm{reg}}} \big(g(X_i) - r_i\big)^2. \label{eq:emp-bonus:erm}
\end{align}
Set $\epsilon_0 \coloneqq m_1^{-1}$, $L_n \coloneqq 2\log N_{\fcl}(\epsilon_0) + \log(8/\delta)$, and $\lambda_n \coloneqq \frac{5 L_n}{3 m_0} + 10 \epsilon_0$. Define the empirical disagreement width, empirical bonus, and clipped pessimistic reward for each $x \in \cS \times \cA$ by
\begin{align}
    \hat{D}_{\fcl, \lambda_n}^2(x) &\coloneqq \sup_{g, h \in \fcl} \frac{\big(g(x) - h(x)\big)^2}{\|g - h\|_{\hat{\nu}_0}^2 + \lambda_n}, \notag\\
    \hat{\Gamma}_n(x) &\coloneqq \sqrt{54\lambda_n}\,\hat{D}_{\fcl, \lambda_n}(x), \qquad \fps(x) \coloneqq \max\big\{\fls(x) - \hat{\Gamma}_n(x), 0\big\}, \label{eq:emp-bonus:def}
\end{align}
and output $\hat{\pi} = \argmax_{\pi \in \Delta(\cA|\cS)} \objfkl(\pi; \fps)$. Computationally, for a general class $\fcl$, this procedure assumes access to (i) a least-squares ERM oracle solving \Cref{eq:emp-bonus:erm} and (ii) an empirical disagreement oracle evaluating $\hat{D}_{\fcl, \lambda_n}^2(x)$ in \Cref{eq:emp-bonus:def}, while the policy update reduces to one-dimensional root-finding per context via \Cref{lem:fkl:opt}.

\begin{proposition}\label{prop:emp-disagreement-fkl}
Under \Cref{assump:full-supp,assume:general-function-approx,assume:poly-covering,assume:single-coverage-bandit}, for any $\delta \in (0, 1)$ and $n \geq 2$, with probability at least $1 - \delta$, the sample-split estimator $\hat{\pi}$ defined via \Cref{eq:emp-bonus:erm,eq:emp-bonus:def} satisfies
\begin{align*}
    \suboptfkl(\hat{\pi}) \leq O\bigg( \frac{\eta C^{\pistar} D_{\pistar}^2}{n} \big( \log N_{\fcl}(n^{-1}) + \log(1/\delta) \big) \bigg),
\end{align*}
and hence achieves $\suboptfkl(\hat{\pi}) \leq \epsilon$ using $n = \tilde{O}\big(\eta C^{\pistar} D_{\pistar}^2 \epsilon^{-1} \log N_{\fcl}(\epsilon)\big)$ samples.
\end{proposition}
\begin{proof}
We first establish a two-sided uniform norm comparison on $I_{\mathrm{geo}}$. Let $\fcl(\epsilon_0) \subset \fcl$ be a minimal $\epsilon_0$-net of $\fcl$ in $\|\cdot\|_\infty$ with $|\fcl(\epsilon_0)| = N_{\fcl}(\epsilon_0)$, and define $\cF_0 \coloneqq \{\tilde{g} - \tilde{h} : \tilde{g}, \tilde{h} \in \fcl(\epsilon_0)\}$, which satisfies $|\cF_0| \leq N_{\fcl}(\epsilon_0)^2$. Fix any $\tilde{f} \in \cF_0$. Since $\fcl \subseteq (\cS \times \cA \to [0, 1])$, we have $|\tilde{f}(X)| \leq 1$, $\tilde{f}(X)^2 \in [0, 1]$, and $\Var_{X \sim \nu}\big(\tilde{f}(X)^2\big) \leq \EE_{X \sim \nu}[\tilde{f}(X)^4] \leq \|\tilde{f}\|_{\nu}^2$. Applying Bernstein's inequality (\Cref{lem:bernstein}) to $\{\tilde{f}(X_i)^2\}_{i \in I_{\mathrm{geo}}}$ and taking a union bound over $\cF_0$ yields that, with probability at least $1 - 2|\cF_0|e^{-L_n} \geq 1 - \delta/4$, simultaneously for all $\tilde{f} \in \cF_0$,
\begin{align}
    \big| \|\tilde{f}\|_{\hat{\nu}_0}^2 - \|\tilde{f}\|_{\nu}^2 \big| \leq \sqrt{\frac{2 L_n \|\tilde{f}\|_{\nu}^2}{m_0}} + \frac{2 L_n}{3 m_0} \leq \frac{1}{2}\|\tilde{f}\|_{\nu}^2 + \frac{5 L_n}{3 m_0}, \label{eq:emp-bonus:bernstein-net}
\end{align}
where the second inequality follows from $2ab \leq \frac{1}{2}a^2 + 2b^2$. Now take arbitrary $g, h \in \fcl$ and set $f \coloneqq g - h$. Choose $\tilde{g}, \tilde{h} \in \fcl(\epsilon_0)$ with $\|g - \tilde{g}\|_\infty \leq \epsilon_0$ and $\|h - \tilde{h}\|_\infty \leq \epsilon_0$, and set $\tilde{f} \coloneqq \tilde{g} - \tilde{h} \in \cF_0$. Then $\|f - \tilde{f}\|_\infty \leq 2\epsilon_0$ and $|f(x) + \tilde{f}(x)| \leq 2$, so $\|f^2 - \tilde{f}^2\|_\infty \leq 4\epsilon_0$, which implies $\big|\|f\|_{\mu}^2 - \|\tilde{f}\|_{\mu}^2\big| \leq 4\epsilon_0$ for both $\mu \in \{\nu, \hat{\nu}_0\}$. Combining this perturbation bound with both directions of \Cref{eq:emp-bonus:bernstein-net} gives, for all $g, h \in \fcl$,
\begin{align}
    \frac{1}{2}\|g - h\|_{\nu}^2 - \lambda_n \leq \|g - h\|_{\hat{\nu}_0}^2 \leq 2\|g - h\|_{\nu}^2 + \lambda_n, \label{eq:emp-bonus:uc}
\end{align}
since $\frac{3}{2}(\|f\|_{\nu}^2 + 4\epsilon_0) + \frac{5L_n}{3m_0} + 4\epsilon_0 \leq 2\|f\|_{\nu}^2 + \lambda_n$ and $\frac{1}{2}(\|f\|_{\nu}^2 - 4\epsilon_0) - \frac{5L_n}{3m_0} - 4\epsilon_0 \geq \frac{1}{2}\|f\|_{\nu}^2 - \lambda_n$.

Independently on $I_{\mathrm{reg}}$, the standard covering-number least-squares bound~\citep[Eq.~(D.2)]{zhao2026towards} guarantees that with probability at least $1 - \delta/2$,
\begin{align}
    \|\fls - \fgt\|_{\nu}^2 \leq \frac{128}{3 m_1}\log\Big(\frac{4 N_{\fcl}(\epsilon_0)}{\delta}\Big) + 18\epsilon_0 \leq \frac{128 L_n}{3 m_0} + 18\epsilon_0 \leq 26 \lambda_n, \label{eq:emp-bonus:erm-bound}
\end{align}
where we used $m_1 \geq m_0$ and $\log(4N_{\fcl}(\epsilon_0)/\delta) \leq L_n$. Let $\cE_{\mathrm{emp}}(\delta)$ denote the intersection event on which both \Cref{eq:emp-bonus:uc,eq:emp-bonus:erm-bound} hold; by a union bound, $\PP(\cE_{\mathrm{emp}}(\delta)) \geq 1 - \delta$. On $\cE_{\mathrm{emp}}(\delta)$, the upper bound in \Cref{eq:emp-bonus:uc} and \Cref{eq:emp-bonus:erm-bound} yield $\|\fls - \fgt\|_{\hat{\nu}_0}^2 + \lambda_n \leq 2\|\fls - \fgt\|_{\nu}^2 + 2\lambda_n \leq 54\lambda_n$. By the definition of $\hat{D}_{\fcl, \lambda_n}$ in \Cref{eq:emp-bonus:def}, $|\fls(x) - \fgt(x)| \leq \sqrt{54\lambda_n}\,\hat{D}_{\fcl, \lambda_n}(x) = \hat{\Gamma}_n(x)$ for all $x \in \cS \times \cA$, whence $\fps(x) \in [0, 1]$ and $0 \leq \fgt(x) - \fps(x) \leq 2\hat{\Gamma}_n(x)$. Meanwhile, the lower bound in \Cref{eq:emp-bonus:uc} gives $\|g - h\|_{\hat{\nu}_0}^2 + \lambda_n \geq \frac{1}{2}\|g - h\|_{\nu}^2$ for all $g, h \in \fcl$, so $\hat{D}_{\fcl, \lambda_n}^2(x) \leq 2 D_{\fcl}^2(x; \piref)$ pointwise and thus $\EE_{X \sim \rho \times \pistarfkl}[\hat{D}_{\fcl, \lambda_n}^2(X)] \leq 2 D_{\pistar}^2$. Invoking \Cref{lem:fkl:mab:perf-diff} context-wise on $\cE_{\mathrm{emp}}(\delta)$ gives
\begin{align*}
    \suboptfkl(\hat{\pi}) &\leq 2\eta C^{\pistar} \EE_{X \sim \rho \times \pistarfkl}\big[4\hat{\Gamma}_n(X)^2\big] \\
    &= 432\eta C^{\pistar} \lambda_n \EE_{X \sim \rho \times \pistarfkl}\big[\hat{D}_{\fcl, \lambda_n}^2(X)\big] \leq 864 \eta C^{\pistar} D_{\pistar}^2 \lambda_n.
\end{align*}
Noting $m_0, m_1 = \Theta(n)$ and converting from statistical error to sample complexity under \Cref{assume:poly-covering} completes the proof.
\end{proof}

\begin{remark}[Tractable instantiations]
In structured classes, $\hat{D}_{\fcl, \lambda}^2(x)$ admits closed-form expressions without calling a generic supremum oracle. In the tabular case $\fcl = (\cS \times \cA \to [0, 1])$, $\hat{D}_{\fcl, \lambda}^2(s, a) = \big(m_0^{-1} N_{\mathrm{geo}}(s,a) + \lambda\big)^{-1} \asymp m_0 / (N_{\mathrm{geo}}(s,a) \vee 1)$ for $\lambda \asymp m_0^{-1}$, recovering the count-based bonus in \Cref{algorithm:bandit-pess-tb}. In the bounded linear class $\fcl = \{x \mapsto \dotp{\bphi(x)}{\btheta} : \btheta \in \bTheta \subset \RR^d\}$ where $\bTheta$ is bounded and $\bTheta - \bTheta$ contains a Euclidean ball of constant radius around the origin, we have $D_{\fcl}^2(x; \piref) \asymp \bphi(x)^\top \bSigma^{\dagger} \bphi(x)$ with $\bSigma \coloneqq \EE_{X \sim \nu}[\bphi(X)\bphi(X)^\top]$, while its empirical counterpart satisfies $\hat{D}_{\fcl, \lambda}^2(x) \asymp \bphi(x)^\top \bar{\bSigma}_{\lambda}^{-1} \bphi(x) = m_0 \bphi(x)^\top \hat{\bSigma}_{m_0 \lambda}^{-1} \bphi(x)$ with $\bar{\bSigma}_{\lambda} \coloneqq m_0^{-1}\sum_{i \in I_{\mathrm{geo}}} \bphi(X_i)\bphi(X_i)^\top + \lambda \Ib_d$ and unnormalized ridge Gram matrix $\hat{\bSigma}_{m_0 \lambda} \coloneqq \sum_{i \in I_{\mathrm{geo}}} \bphi(X_i)\bphi(X_i)^\top + m_0 \lambda \Ib_d$ (so $\hat{\Gamma}_n(x) \asymp \sqrt{m_0 \lambda_n}\,\|\bphi(x)\|_{\hat{\bSigma}_{m_0 \lambda_n}^{-1}}$), matching standard elliptical confidence bounds~\citep{jin2020provably,zhu2023principled}.
\end{remark}

\subsection{Simultaneous $\eta$-Independent Slow-Rate Upper Bounds}\label{app:subsec:slow-rate}

Although \Cref{thm:tabular,thm:function-approximation} focus on the $\tilde{O}(\epsilon^{-1})$ fast rate, the exact same output policies $\hat{\pi}$ of \Cref{algorithm:bandit-pess-tb,algorithm:bandit-pess-gfa} simultaneously enjoy $\eta$-independent $\tilde{O}(\epsilon^{-2})$ slow-rate guarantees on the same high-probability events.

\begin{proposition}\label{prop:simultaneous-slow-rate}
Under \Cref{assump:full-supp}, for any $\eta > 0$ and $\delta \in (0, 1)$:
\begin{enumerate}
    \item On $\confbandittb(\delta) \cap \confbanditbin(\delta)$, the output $\hat{\pi}$ of \Cref{algorithm:bandit-pess-tb} simultaneously satisfies
    \begin{align*}
        \suboptfkl(\hat{\pi}) \leq \tilde{O}\bigg( \frac{\eta (C^{\pistar})^2 |\cS||\cA|}{n} \wedge \sqrt{\frac{C^{\pistar}|\cS||\cA|}{n}} \bigg).
    \end{align*}
    \item Under \Cref{assume:general-function-approx,assume:poly-covering,assume:single-coverage-bandit} with $\epsilon_c = n^{-1}$ in \Cref{eq:bandit-pess-function-approximation}, on $\confbanditfa(\delta)$, the output $\hat{\pi}$ of \Cref{algorithm:bandit-pess-gfa} simultaneously satisfies
    \begin{align*}
        \suboptfkl(\hat{\pi}) \leq \tilde{O}\bigg( \frac{\eta C^{\pistar} D_{\pistar}^2 \log N_{\fcl}(n^{-1})}{n} \wedge D_{\pistar}\sqrt{\frac{\log N_{\fcl}(n^{-1})}{n}} \bigg).
    \end{align*}
\end{enumerate}
\end{proposition}
\begin{proof}
Since $\hat{\pi} \in \argmax_{\pi \in \Delta(\cA|\cS)} \objfkl(\pi; \fps)$ satisfies $\objfkl(\hat{\pi}; \fps) \geq \objfkl(\pistarfkl; \fps)$, the linear decomposition in \Cref{eq:lem:fkl:perf-diff-main-ineq} holds context-wise for every $\eta > 0$. On the respective pessimism events, $\fgt(s,a) - \fps(s,a) \geq 0$ for all $(s,a) \in \cS \times \cA$, and hence dropping the non-positive term $-\dotp{\fgt - \fps}{\hat{\pi}} \leq 0$ yields
\begin{align}
    \suboptfkl(\hat{\pi}) \leq \dotp{\fgt - \fps}{\pistarfkl - \hat{\pi}} \leq \EE_{(s,a) \sim \rho \times \pistarfkl}\big[ \fgt(s,a) - \fps(s,a) \big]. \label{eq:slow-rate:master}
\end{align}
In the tabular setting, on $\confbandittb(\delta) \cap \confbanditbin(\delta)$, the proof of \Cref{thm:tabular} shows $(\fgt(s,a) - \fps(s,a))^2 \leq \frac{128\log^2(2|\cS||\cA|/\delta)}{n\rho(s)\piref(a|s)}$ for all $(s,a) \in \cS \times \cA$. Applying Cauchy--Schwarz to \Cref{eq:slow-rate:master} gives
\begin{align*}
    \suboptfkl(\hat{\pi}) &\leq \frac{8\sqrt{2}\log(2|\cS||\cA|/\delta)}{\sqrt{n}} \sum_{s \in \cS, a \in \cA} \sqrt{\rho(s)\pistarfkl(a|s) \frac{\pistarfkl(a|s)}{\piref(a|s)}} \\
    &\leq 8\sqrt{2}\log(2|\cS||\cA|/\delta) \sqrt{\frac{|\cS||\cA|}{n} \sum_{s \in \cS, a \in \cA} \rho(s)\pistarfkl(a|s)\frac{\pistarfkl(a|s)}{\piref(a|s)}} \\
    &= \tilde{O}\bigg(\sqrt{\frac{C^{\pistar}|\cS||\cA|}{n}}\bigg).
\end{align*}
Under function approximation, on $\confbanditfa(\delta)$, we have $0 \leq \fgt(s,a) - \fps(s,a) \leq 2\beta D_{\fcl}((s,a); \piref)$. Applying Jensen's inequality to \Cref{eq:slow-rate:master} yields $\suboptfkl(\hat{\pi}) \leq 2\beta \EE_{(s,a) \sim \rho \times \pistarfkl}[D_{\fcl}((s,a); \piref)] \leq 2\beta D_{\pistar}$. Combining these $\eta$-free bounds with the fast-rate bounds in \Cref{thm:tabular,thm:function-approximation} completes the proof.
\end{proof}

\Cref{prop:simultaneous-slow-rate} implies simultaneous sample complexity bounds of $\tilde{O}\big(\eta (C^{\pistar})^2 |\cS||\cA|\epsilon^{-1} \wedge C^{\pistar}|\cS||\cA|\epsilon^{-2}\big)$ in the tabular setting and $\tilde{O}\big(\eta C^{\pistar} D_{\pistar}^2 \epsilon^{-1}\log N_{\fcl}(\epsilon) \wedge D_{\pistar}^2 \epsilon^{-2}\log N_{\fcl}(\epsilon^2)\big)$ under function approximation. In the tabular setting, this upper bound transitions at $\eta \asymp (C^{\pistar}\epsilon)^{-1}$ (or $\eta \asymp \sqrt{n / ((C^{\pistar})^3 |\cS||\cA|)}$ in sample size $n$), whereas our two-armed lower bound in \Cref{thm:fkl:mab:lb-ref=unif} transitions at $\eta \asymp \sqrt{n}$; reconciling these transition thresholds across general $|\cS|$, $|\cA|$, and $C^{\pistar}$ remains an open problem.

\subsection{Scope of the Convex-Conjugate Pipeline Beyond Forward KL and Offline Settings}\label{app:subsec:beyond-fkl}

We examine why the specific generator of forward KL and offline pessimism are pivotal for achieving the $\epsilon^{-1}$ fast rate in \Cref{lem:fkl:mab:perf-diff}, and identify the exact technical barriers that arise for general strictly convex regularizers and online regret minimization.

\noindent\textbf{General strictly convex $f$-divergences.} Consider a single-context $f$-divergence regularizer $h(\pi) \coloneqq \eta^{-1} D_f(\pi \| \piref) = \eta^{-1}\sum_{a \in \cA} \piref(a) f\big(\pi(a)/\piref(a)\big)$, where $f: (0, +\infty) \to \RR$ is twice continuously differentiable and strictly convex with $f(1) = 0$. Note that forward KL corresponds to $D_f(\pi \| \piref)$ with $f(t) = -\log t$. Assuming $\pistar = \pi_r \in \ri(\Delta(\cA))$ and $\hat{\pi} = \pi_{\hat{r}} \in \ri(\Delta(\cA))$, \Cref{lem:simplex-conjugate-duality,lem:simplex-conjugate-bregman} still yield $\subopt_{D_f}(\hat{\pi}) = \breg{h^*}{r}{\hat{r}} = \breg{h}{\hat{\pi}}{\pistar}$. Writing $x_a \coloneqq \pistar(a)/\piref(a) > 0$ and $y_a \coloneqq 1 - \hat{\pi}(a)/\pistar(a) < 1$ so that $\hat{\pi}(a)/\piref(a) = x_a(1 - y_a)$, a direct expansion of the Bregman divergence gives the analogue of \Cref{eq:lem:fkl:perf-diff-main-eq}:
\begin{align*}
    \subopt_{D_f}(\hat{\pi}) = \eta^{-1} \sum_{a \in \cA} \piref(a) \phi_{x_a}(y_a), \quad \text{where } \phi_x(y) \coloneqq f\big(x(1 - y)\big) - f(x) + x f'(x) y.
\end{align*}
Combining this identity with $\subopt_{D_f}(\hat{\pi}) \leq \sum_{a \in \cA} \pistar(a)(r(a) - \hat{r}(a)) y_a$ and applying the Fenchel--Young inequality action-wise with $z_a \coloneqq 2\eta \frac{\pistar(a)}{\piref(a)}(r(a) - \hat{r}(a))$ yields
\begin{align}
    \subopt_{D_f}(\hat{\pi}) \leq \eta^{-1} \sum_{a \in \cA} \piref(a) \phi_{x_a}^*(z_a), \quad \text{where } \phi_x^*(z) \coloneqq \sup_{y < 1}\big\{ yz - \phi_x(y) \big\}. \label{eq:beyond-fkl:general-bound}
\end{align}
For $f(t) = -\log t$, logarithmic homogeneity yields $\phi_x(y) = -\log(1 - y) - y$, which decouples completely from $x_a = \pistar(a)/\piref(a)$ and satisfies $\phi''(0) = 1 > 0$, so its conjugate $\phi^*(z) = z - \log(1 + z)$ enjoys the quadratic bound $\phi^*(z) \leq z^2/2$ for all $z \geq 0$. For a general strictly convex generator $f$ lacking strong convexity, $\phi_{x_a}$ depends on $x_a$ and can have vanishing second derivative $\phi_x''(0) = x^2 f''(x) = 0$, destroying the $O(z^2)$ dual bound needed to obtain an $\epsilon^{-1}$ fast rate via \Cref{eq:beyond-fkl:general-bound}. For example, take the strictly convex generator $f(t) = (t - 1)^4$. At any action with $x_a = \pistar(a)/\piref(a) = 1$, we have $\phi_1(y) = y^4$, whose constrained conjugate over $y \leq 1$ evaluates to $\phi_1^*(z) = 3 \cdot 4^{-4/3} z^{4/3}$ for $z \in [0, 4]$. Since $\phi_1^*(z) = \Theta(z^{4/3})$ rather than $O(z^2)$ as $z \downarrow 0$, \Cref{eq:beyond-fkl:general-bound} yields only an $O(n^{-2/3})$ rate (i.e., $O(\epsilon^{-3/2})$ sample complexity) rather than an $O(\epsilon^{-1})$ fast rate.

\noindent\textbf{Obstacles in online regret minimization.} While our lower bounds in \Cref{sec:lb} extend naturally to online forward-KL-regularized regret~\citep{ji2026near}, extending our upper-bound pipeline in \Cref{lem:fkl:mab:perf-diff} to online optimism is nontrivial even for multi-armed bandits unless one allows an explicit $(\min_a \piref(a))^{-1}$ factor via \Cref{lem:fkl:sc}. Specifically, in \Cref{eq:fkl:perf-diff-self-bounding,eq:lem:fkl:perf-diff-direct-pess}, pessimism ($r \geq \hat{r}$) guarantees $z_a \geq 0$, keeping $z_a$ bounded away from the singularity $z = -1$ of $\phi^*(z) = z - \log(1 + z)$ and enabling $\phi^*(z_a) \leq z_a^2/2$. Under online optimism ($\hat{r} \geq r$), however, $z_a \leq 0$, so $\phi^*(z_a)$ blows up as $z_a \downarrow -1$ (and becomes $+\infty$ whenever $z_a \leq -1$), precluding a direct application of \Cref{lem:fkl:mab:perf-diff} without new algorithmic mechanisms.

\section{Technical Tools}

\subsection{Differentiating maximal functions}

We adapt simplified versions of \citet[Proposition~4.5.1]{bertsekas2003convex} and \citet[Theorem~1]{milgrom2002envelope} as follows.
\begin{lemma}[Danskin's Theorem, {\citealt{bertsekas2003convex}}]\label{lem:danskin}
    Let $\Pi \subset \RR^m$ be compact, $G: \RR^n \times \Pi \to \RR$ be continuous, and $\forall \pi \in \Pi, r \mapsto G(r, \pi)$ be convex. For $g(\cdot) \coloneqq \max_{\pi \in \Pi} G(\cdot, \pi)$,
    if $\argmax_{\pi \in \Pi} G(\tilde{r}, \pi) = \{\tilde{\pi}\}$ is a singleton and $r \mapsto G(r, \tilde{\pi})$ is differentiable at $r = \tilde{r}$, then $\nabla g(\tilde{r})$ exists and $\nabla g(\tilde{r}) = \nabla_1 G(\tilde{r}, \tilde{\pi})$, where $\nabla_1$ is the gradient with respect to the first input.
\end{lemma}

\begin{lemma}[The Envelope Theorem, {\citealt{milgrom2002envelope}}]\label{lem:envelope}
    Let $\Pi$ be a set and $J: \Pi \times [0, 1] \to \RR$ be a function. Let $V(\cdot) \coloneqq \sup_{\pi \in \Pi} J(\pi, \cdot)$. For any $u \in (0, 1)$ and $\tilde{\pi} \in \Pi$, if $J(\tilde{\pi}, u) = V(u)$, $V'(u)$ exists, and $\partial_u J(\tilde{\pi}, u)$ exists; then $V'(u) = \partial_u J(\tilde{\pi}, u)$.
\end{lemma}

\subsection{Properties of convex conjugate over the simplex}\label{app:prop-conjugate-over-simplex}

The \emph{convex conjugate}, whose supremum is taken over the entire Euclidean space, has been a well-established notion~\citep{rockafellar1997convex}; it has various handy properties for the analysis of bandit learning (see, e.g., \citet[Chapter~26]{lattimore2020bandit}). We adapt some of them to the case where the supremum is taken over the probability simplex $\Delta(\cA)$ for the sake of rigorousness.

\begin{lemma}\label{lem:simplex-conjugate-duality}
Suppose a closed convex function $h: \Delta(\cA) \to \RR \cup \{+\infty\}$ is continuously differentiable and strictly convex on $\ri(\Delta(\cA))$, and $\mathrm{dom}(h) \supset \ri(\Delta(\cA))$. Suppose for any $r \in \RR^\cA$, $\pi_r \in \ri(\Delta(\cA))$ exists, where $\pi_r \coloneqq \argmax_{\pi \in \Delta(\cA)} \dotp{r}{\pi} - h(\pi)$.
Then $h^*(\cdot) \coloneqq \max_{\pi \in \Delta(\cA)} \dotp{\cdot}{\pi} - h(\pi)$ is differentiable and $\nabla h^*(r) = \pi_r$.
\end{lemma}
\begin{proof}[Proof of \Cref{lem:simplex-conjugate-duality}]
First, note that $\pi_r$ in the relative interior is uniquely well-defined by the existence assumption and strict convexity. Also, $h$ induces another proper, closed, and convex function $\tilde{h}: \RR^{\cA} \to (-\infty, +\infty]$ defined as
\begin{align*}
    \tilde{h}(\pi) \coloneqq \begin{cases}
        h(\pi) , &\pi \in \Delta(\cA); \\
        +\infty ,& \text{otherwise};
    \end{cases}
\end{align*}
which is also bounded from below~\citep[Theorem~2.12]{beck2017first}. Here, the (unconstrained) conjugate of $\tilde{h}$ evaluates to $h^*$ because
\begin{align*}
    \tilde{h}^*(r) = \sup_{\pi \in \RR^{\cA}} \dotp{r}{\pi} - \tilde{h}(\pi) = \max_{\pi \in \Delta(\cA)} \dotp{r}{\pi} - {h}(\pi) = h^*(r),
\end{align*}
which, together with the compactness of $\Delta(\cA)$, implies that $h^*$ is finite everywhere. Also, since $h^* = \tilde{h}^*$ implies that $h^*$ is convex~\citep[Theorem~4.3]{beck2017first}. Therefore, the (assumed) uniqueness of $\pi_r = \argmax_{\mu \in \RR^\cA} \dotp{r}{\mu} - \tilde{h}(\mu)$ applying to the conjugate subgradient theorem~\citep[Theorem~4.20]{beck2017first} yields that
\begin{align}
    \pi_r \in \partial \tilde{h}^*(r) = \partial h^*(r) = \{\pi_r\}. \label{eq:unique-suggradient}
\end{align}
Since we already shown $h^*$ is finite everywhere, applying \citet[Theorem~3.33]{beck2017first} to \Cref{eq:unique-suggradient} yields that $h^*$ is differentiable and $\nabla h^*(r) = \pi_r$.
\end{proof}

\begin{lemma}\label{lem:simplex-conjugate-bregman}
    Under the conditions (and following the notation) in \Cref{lem:simplex-conjugate-duality}, $\forall r, r' \in \RR^\cA$,
\begin{align*}
    \breg{h^*}{r}{r'} = \breg{h}{\pi_{r'}}{\pi_r},
\end{align*}
where the Bregman divergence induced by $h$ is defined as $\breg{h}{\pi_{r'}}{\pi_r} \coloneqq h(\pi_{r'}) -h(\pi_{r}) - \dotp{\nabla h(\pi_{r})}{\pi_{r'} - \pi_r}$ and similarly for $h^*$.
\end{lemma}
\begin{proof}
    Direct calculates yield
\begin{align}
    \text{LHS} &= h^*(r) - h^*(r') - \dotp{\nabla h^*(r')}{ r - r'} \notag \\
    &= h^*(r) - h^*(r') - \dotp{ \pi_{r'} }{ r - r'} \notag \\
    &= \big( \dotp{r}{\pi_r} - h(\pi_r) \big) - \big( \dotp{r'}{\pi_{r'}} - h(\pi_{r'}) \big) - \dotp{\pi_{r'}}{r} +  \dotp{\pi_{r'}}{r'} \notag \\
    &= h(\pi_{r'}) - h(\pi_{r}) - \dotp{r}{\pi_{r'}  - \pi_{r}}, \label{eq:simplex-conjugate-bregman-half}
\end{align}
where the second and third equalities follow from \Cref{lem:simplex-conjugate-duality}. Because $\pi_r \in \ri(\Delta(\cA))$ maximizes $\dotp{r}{\cdot} - h(\cdot)$ over $\Delta(\cA)$, the first-order optimality condition (of the Lagrangian) under the normalization constraint (i.e., probabilities sum to one) implies that $\exists \lambda \in \RR$ such that
\begin{align*}
  \forall a \in \cA,  r(a) - \frac{\partial h}{\partial \pi(a)}(\pi_r) = \lambda,
\end{align*}
where we note that there is no Lagrangian multiplier (dual variable) for the non-negativity constraint due to complementary slackness and the assumption that $\pi_r \in \ri(\Delta(\cA))$. Thus, $r - \nabla h(\pi_r) \in \mathrm{span}(\one)$.\footnote{Recall that this $\mathrm{span}(\one)$ is inherent, as elaborated in \Cref{rmk:continuation}.} This relation implies $\dotp{r}{\pi_{r'}  - \pi_{r}} = \dotp{\nabla h(\pi_r)}{\pi_{r'}  - \pi_{r}}$ and in turn implies
\begin{align*}
    \Cref{eq:simplex-conjugate-bregman-half} = h(\pi_{r'}) - h(\pi_{r}) - \dotp{\nabla h(\pi_r)}{\pi_{r'}  - \pi_{r}} = \text{RHS}.
\end{align*}
\end{proof}

\begin{remark}\label{rmk:continuation}
    To be conceptually minimal (without the introduction any convex-analytical definitions of \emph{gradient} for functions defined on low-dimensional convex subsets), we understand the ``continuous differentiability'' for $h$ in \Cref{lem:simplex-conjugate-duality,lem:simplex-conjugate-bregman} as: $h$ admits a \emph{continuation} $h_\mathrm{ext}: \RR_+^{\cA} \to (-\infty, +\infty]$ such that the \emph{Euclidean gradient} $\mu \mapsto \nabla h_\mathrm{ext}(\mu)$ is continuous. The continuation $h_\mathrm{ext}$ is not necessarily unique, but this non-uniqueness does not hurt in the proof of either \Cref{lem:simplex-conjugate-duality} or \Cref{lem:simplex-conjugate-bregman} because if there are two different valid continuations $h_\mathrm{ext}^{(1)}$ and $h_\mathrm{ext}^{(2)}$, they must agree on $\Delta(\cA)$, viz., their directional derivatives along any path strictly inside $\Delta(\cA)$ must be equal, i.e., $\forall \pi \in \ri(\Delta(\cA))$, $\dotp{\nabla h_\mathrm{ext}^{(1)}(\pi)}{\tb} = \dotp{\nabla h_\mathrm{ext}^{(2)}(\pi)}{\tb} $ as long as the tangent direction $\tb$ satisfies $\dotp{\tb}{\one} = 0$; in which case $\nabla h_\mathrm{ext}^{(1)}(\pi) - \nabla h_\mathrm{ext}^{(2)}(\pi) \in \mathrm{span}(\one)$, and thus the first-order quantity $\dotp{\nabla h_\mathrm{ext}(\pi)}{\pi' - \pi}$ between any two probability vectors $\pi'$ and $\pi$ does NOT vary with the choice of continuation. We thus only write $\nabla h$ in place of $\nabla h_\mathrm{ext}$ in these lemma statements and proofs.
\end{remark}
\begin{remark}
    The $\mathrm{dom}(h) \supset \ri(\Delta(\cA))$ in \Cref{lem:simplex-conjugate-duality} implies that we allow $h$ to blow up on the boundary of $\Delta(\cA)$, which is the case for forward KL $h(\cdot) = \eta^{-1} \kl{\piref}{\cdot}$ even under \Cref{assump:full-supp}. This is the reason that prevents us from using Danskin's theorem (\Cref{lem:danskin}) to directly prove \Cref{lem:simplex-conjugate-duality}.
\end{remark}

\begin{lemma}\label{lem:simplex-conjugate-hessian}
Under the conditions (and following the notation) in \Cref{lem:simplex-conjugate-duality}, if $h$ is also twice continuously differentiable on $\ri(\Delta(\cA))$ and there is an $r_\flat \in \RR^\cA$ such that the Hessian matrix $\nabla^2 h(\pi_{r_\flat}) \succ \zero_{A \times A}$, then $h^*$ is twice continuously differentiable in a neighborhood of $r_\flat$ and
\begin{align}
    \big[ \nabla^2 h^*(r_\flat) \big] = \big[ \nabla^2 h(\pi_{r_\flat}) \big]^{-1} - \frac{\big[ \nabla^2 h(\pi_{r_\flat}) \big]^{-1} \one \one^\top \big[ \nabla^2 h(\pi_{r_\flat}) \big]^{-1}}{\one^\top \big[ \nabla^2 h(\pi_{r_\flat}) \big]^{-1} \one}, \label{eq:simplex-conjugate-hessian-invert}
\end{align}
where $\one \in \RR^\cA$ is the all-one vector. Consequently, let $\Pb \coloneqq \Ib - A^{-1}\one\one^\top$,\footnote{$\Pb$ is the the orthogonal projection matrix onto the tangent space of $\Delta(\cA)$ at any $\pi \in \Delta(\cA)$.} \Cref{eq:simplex-conjugate-hessian-invert} implies
\begin{align*}
    \big[ \nabla^2 h^*(r_\flat) \big]\cdot \big[ \nabla^2 h(\pi_{r_\flat}) \big] \cdot \Pb = \Pb.
\end{align*}
\end{lemma}

\begin{proof}
     Let $\Hb \coloneqq \big[ \nabla^2 h(\pi_{r_\flat}) \big]$. Then $\Hb \succ \zero_{A \times A}$ by assumption. By \Cref{lem:simplex-conjugate-duality}, $h^*$ is a closed convex function that is differentiable everywhere and $\nabla h^*(r) = \pi_r$, and thus \Cref{lem:differentiable-to-C-1--boosting} implies that $r \mapsto \nabla h^*(r)$ is continuous on $\RR^\cA$. Since $\pi_{r_\flat} \in \ri\big( \Delta(\cA) \big)$, there is an open neighborhood $\cU \subset \RR^\cA$ of $r_\flat$ such that $\forall r \in \cU$, $\pi_r \in \ri\big(\Delta(\cA)\big)$ and the non-negativity constraint is inactive due to complementary slackness. And thus the first-order optimality condition and the primal feasibility $\one^\top \pi = 1$ are necessary and sufficient for characterizing the maximizer $\pi_r$ (unique by assumption) for every $r \in \cU$, i.e., there exists a unique Lagrangian multiplier $\lambda_r$ for each $r \in \cU$ such that
\begin{align*}
\begin{cases}
    \nabla h(\pi_r) - r - \lambda_r \one &= \zero, \\
    \inner{\one}{\pi_r} - 1 &= 0.
\end{cases}
\end{align*}
Let the mapping $\fb: \RR^A \times \RR \times \RR^A \to \RR^{A+1}$ be
\begin{align*}
    \fb(\pi, \lambda, r) \coloneqq \begin{bmatrix}
        \nabla h(\pi) - r - \lambda \one \\
        \inner{\one}{\pi} -  1.
    \end{bmatrix}
\end{align*}
Then $\fb(\pi_{r_\flat}, \lambda_{r_\flat}, r_\flat) = \zero_{A+1}$. The Jacobian matrix $\Jb_\flat$ of $\fb$ with respect to the variables $(\pi, \lambda)$ (i.e., excluding the explicit dependency of $r$) evaluated at $(\pi_{r_\flat}, \lambda_{r_\flat}, r_\flat)$ is thus
\begin{align*}
    \Jb_\flat = \begin{bmatrix}
        \Hb & -\one_A \\
        \one_A^\top & 0
    \end{bmatrix}.
\end{align*}
To verify that $\Jb_\flat$ is non-singular, suppose $\Jb_\flat [\ab^\top\ b]^\top = \zero_{A+1}$ for some $\ab \in \RR^A$ and $b \in \RR$, which is equivalent to $\Hb \ab - b \one_A = \zero_A$ and $\one_A^\top \ab = 0$. Since $\Hb \succ \zero_{A \times A}$ is invertible with $\Hb^{-1} \succ \zero_{A \times A}$, left-multiplying $\Hb \ab = b \one_A$ by $\one_A^\top \Hb^{-1}$ yields $0 = \one_A^\top \ab = b (\one_A^\top \Hb^{-1} \one_A)$. Because $\one_A \neq \zero_A$ and $\Hb^{-1} \succ \zero_{A \times A}$, we have $\one_A^\top \Hb^{-1} \one_A > 0$, which forces $b = 0$ and in turn $\Hb \ab = \zero_A \implies \ab = \zero_A$. Thus $\Jb_\flat$ has trivial kernel and is non-singular; applying the implicit function theorem on which gives that $r \mapsto (\pi_r, \lambda_r)$ is continuously differentiable in an open neighborhood $\tilde{\cU} \subset \RR^\cA$ of $r_\flat$. Recall that $\nabla h^*(r) = \pi_r$, we further obtain that $r \mapsto \nabla h^*(r)$ is continuously differentiable locally, viz., $h^*$ is twice continuously differentiable at $r_\flat$. Therefore, it is legal to calculate $\nabla^2 h^*(r_\flat)$. Recall that for $r \in \cU$, $\fb(\pi_r, \lambda_r, r) \equiv \zero_{A+1}$, we use the chain rule for taking the derivative on both sides of $\fb(\pi_r, \lambda_r, r) = \zero_{A+1}$ with respect to $r$ and evaluate the derivative at $r = r_\flat$ to obtain
\begin{align*}
    \begin{bmatrix}
        \Hb & -\one \\
        \one^\top & 0
    \end{bmatrix} \begin{bmatrix}
        \nabla_r \pi_{r_\flat} \\
        \nabla_r \lambda_{r_\flat}
    \end{bmatrix} + \begin{bmatrix}
        -\Ib_{A\times A} \\
        \zero_A^\top
    \end{bmatrix} = \zero_{(A+1)\times A},
\end{align*}
which implies $\one^\top \nabla_r \pi_{r_\flat} = \zero_A^\top$ and $\Hb \nabla_r \pi_{r_\flat} - \one \nabla_r \lambda_{r_\flat} = \Ib_{A \times A}$.\footnote{In this equation we use the row-layout for notational simplicity, i.e., $\nabla_r \lambda_{r_\flat} \in \RR^{1 \times A}$.} Thus, $\nabla_r \pi_{r_\flat} = \Hb^{-1} + \Hb^{-1} \one \nabla_r \lambda_{r_\flat}$, substituting which into $\one^\top \nabla_r \pi_{r_\flat} = \zero^\top$ yields
\begin{align*}
  \zero_A =  \one^\top \big( \Hb^{-1} + \Hb^{-1} \one \nabla_r \lambda_{r_\flat} \big) \Longrightarrow \nabla_r \lambda_{r_\flat} = - \frac{\one^\top \Hb^{-1}}{ \one^\top \Hb^{-1} \one }.
\end{align*}
Finally,
\begin{align*}
    \nabla^2 h^*(r_{\flat}) = \nabla_r \pi_{r_\flat} = \Hb^{-1} + \Hb^{-1} \one \nabla_r \lambda_{r_\flat} =  \Hb^{-1}  - \frac{ \Hb^{-1} \one\one^\top \Hb^{-1}}{ \one^\top \Hb^{-1} \one }.
\end{align*}
\end{proof}

\begin{remark}
    \Cref{lem:simplex-conjugate-hessian} is a fine-grained version of \citet[Lemma~E.4]{zhao2026towards}, which is used in the current paper only to demonstrate why the mean-value-type argument fails to achieve single-policy concentrability under forward-KL regularization, but can be of independent interest. The counterpart of \Cref{lem:simplex-conjugate-bregman} for \emph{unconstrained} convex conjugate has been established before (see, e.g., \citet[Lemma~26.4]{lattimore2020bandit}).
\end{remark}

\subsection{Miscellanies}

\begin{lemma}[{\citealt[Corollary~25.5.1]{rockafellar1997convex}}]\label{lem:differentiable-to-C-1--boosting}
 Suppose $h$ is a proper convex function that is finite and differentiable on an open convex set $C$, then $h$ is continuously differentiable on $C$.
\end{lemma}
 
\begin{fact}\label{fact:fkl-phi-convex-conjugate}
    $\phi(y) = -y - \log(1-y)$ is convex on $(-\infty, 1)$, whose (constrained) convex conjugate is
\begin{align*}
    \phi^*(z) \coloneqq \sup_{y<1} yz - \phi(y) = \begin{cases}
        z - \log(1+z), &z > -1;\\
        +\infty, &z \leq -1.
    \end{cases}
\end{align*}
And the Fenchel-Young inequality $yz \leq \phi(y) + \phi^*(z)$ holds for any  $y < 1$ and $z \in \RR$.
\end{fact}
\begin{proof}
Direct calculation gives that $\phi'(y) = -1 + (1-y)^{-1}$, thus $\phi''(y) = (1-y)^{-2} > 0$ for all $y < 1$, concluding that $\phi$ is convex. Now fix any $z > -1$, let $\phi_z(y) = yz - \phi(y)$ for $y < 1$, then
\begin{align*}
    \phi_z'(y) = 1 +z - \frac{1}{1-y}.
\end{align*}
Therefore $\phi_z$ takes maximum at $z(1+z)^{-1}$. Plugging in this value of $y$ gives that $\phi^*(z) = z - \log(1+z)$. Finally, for any $y < 1$ and $z \geq 0$, $\phi^*(z) \geq yz - \phi(y)$, resulting in $yz \leq \phi(y) + \phi^*(z)$.
\color{black}
\end{proof}

\begin{lemma}[Lemma A.1, \citealt{xie2021policy}]\label{lem:binary-concentration}
Suppose $N \sim \mathsf{Bin}(n,p)$ where $n \geq 1$ and $p >0$, then
\begin{align*}
    \PP\Big[N \geq \frac{np}{2}\Big] \geq 1 - \exp(-np/8).
\end{align*}
Equivalently, with probability at least $1 - \delta$,
\begin{align*}
    \frac{p}{N \vee 1} \leq \frac{8\log(1/\delta)}{n}.
\end{align*}
\end{lemma}

\begin{lemma}[Subgaussian Hoeffding inequality, {\citealt[Proposition~2.5]{wainwright2019high}}]\label{lem:azuma-hoeffding}
Let $\varepsilon_1, \dots, \varepsilon_m$ be independent mean-zero $\sigma$-subgaussian random variables, i.e., $\EE[\exp(\lambda \varepsilon_j)] \leq \exp(\lambda^2 \sigma^2 / 2)$ for all $\lambda \in \RR$ and $j \in [m]$. Then for every $t \geq 0$,
\begin{align*}
    \PP\bigg( \Big| \frac{1}{m}\sum_{j=1}^m \varepsilon_j \Big| \geq t \bigg) \leq 2\exp\Big( -\frac{m t^2}{2\sigma^2} \Big).
\end{align*}
\end{lemma}

\begin{lemma}[Bernstein's inequality, {\citealt[Proposition~2.10]{wainwright2019high}}]\label{lem:bernstein}
Let $Y_1, \dots, Y_m$ be $\iid$ random variables taking values in $[0, 1]$ with $\sigma^2 \coloneqq \Var(Y_1)$. Then for every $t > 0$, with probability at least $1 - 2e^{-t}$,
\begin{align*}
    \bigg| \frac{1}{m}\sum_{i=1}^m Y_i - \EE[Y_1] \bigg| \leq \sqrt{\frac{2\sigma^2 t}{m}} + \frac{2t}{3m}.
\end{align*}
\end{lemma}

\begin{lemma}[Le Cam's two-point method, \citealt{lecam1973convergence,yu1997assouad}]\label{lem:two-point}
    Let $\cR$ be the set of instances, $\Pi$ be the set of estimators, and $L:\Pi \times \cR \to \RR_+$ be a loss function. For $\acute{r}, \grave{r} \in \cR$, suppose $\exists c > 0$ such that
\begin{align*}
    \inf_{\pi \in \Pi} L(\pi, \acute{r}) + L(\pi, \grave{r}) \geq c,
\end{align*}
then 
\begin{align*}
    \inf_{\pi \in \Pi} \sup_{r \in \cR} \EE_{\cD \sim P_r} L\big( \pi(\cD), r \big) \geq \frac{c}{4}\cdot \exp\big(-\kl{P_{\acute{r}}}{P_{\grave{r}}} \big),
\end{align*}
where the trajectory distribution of $\pi$ interacting with instance $r$ is denoted by $P_r$.
\end{lemma}

\begin{lemma}[Assouad's Lemma, \citealt{yu1997assouad,chen2024assouad}]\label{lem:assouad}
   Let $\cR$ be the set of instances, $\Pi$ be the set of estimators, $\Theta \coloneqq \{\pm1\}^S$ for some $S > 0$, and $\{L_j\}_{j=1}^S$ be $S$ functions from $\Pi \times \cR$ to $\RR_+$. Suppose $\{r_\theta\}_{\theta \in \Theta} \subset \cR$ and the loss function is
\begin{align*}
    L(\pi, r) \coloneqq \sum_{j=1}^S L_j(\pi, r), \forall (\pi, r) \in \Pi \times \cR.
\end{align*}
We denote $\theta \sim_j \theta'$ if they differ only in the $j$-th coordinate. Further assume that
\begin{align}
    \theta \sim_j \theta' \Rightarrow \inf_{\pi\in\Pi} L_j(\pi, r_\theta) + L_j(\pi, r_{\theta'}) \geq c
\end{align}
for some $c > 0$, then
\begin{align*}
    \inf_{\pi \in \Pi} \sup_{r \in \cR} \EE_{\cD \sim P_r} L\big( \pi(\cD), r \big) \geq S \cdot \frac{c}{4} \min_{\exists j: \theta \sim_j \theta'} \exp\Big( - \kl{P_{r_\theta}}{P_{r_{\theta'}}} \Big),
\end{align*}
where the trajectory distribution of $\pi$ interacting with instance $r \in \cR$ is denoted by $P_r$.
\end{lemma}


\newpage
\section*{NeurIPS Paper Checklist}

\begin{enumerate}

\item {\bf Claims}
    \item[] Question: Do the main claims made in the abstract and introduction accurately reflect the paper's contributions and scope?
    \item[] Answer: \answerYes{} 
    \item[] Justification: We clearly state the problem of forward-KL-regularized contextual bandits settled in this paper and the corresponding sample complexity results, at least informally, together with the implications of our bounds. We also discuss in the introduction the role of our work in the line of efforts towards finally figuring out sample complexity of decision making problem with regularized objectives.
    \item[] Guidelines:
    \begin{itemize}
        \item The answer \answerNA{} means that the abstract and introduction do not include the claims made in the paper.
        \item The abstract and/or introduction should clearly state the claims made, including the contributions made in the paper and important assumptions and limitations. A \answerNo{} or \answerNA{} answer to this question will not be perceived well by the reviewers. 
        \item The claims made should match theoretical and experimental results, and reflect how much the results can be expected to generalize to other settings. 
        \item It is fine to include aspirational goals as motivation as long as it is clear that these goals are not attained by the paper. 
    \end{itemize}

\item {\bf Limitations}
    \item[] Question: Does the paper discuss the limitations of the work performed by the authors?
    \item[] Answer: \answerYes{} 
    \item[] Justification: We detail in the last paragraph of our conclusion that the necessity of single-policy concentrability for learning forward-KL-regularized contextual bandits is still open.
    \item[] Guidelines:
    \begin{itemize}
        \item The answer \answerNA{} means that the paper has no limitation while the answer \answerNo{} means that the paper has limitations, but those are not discussed in the paper. 
        \item The authors are encouraged to create a separate ``Limitations'' section in their paper.
        \item The paper should point out any strong assumptions and how robust the results are to violations of these assumptions (e.g., independence assumptions, noiseless settings, model well-specification, asymptotic approximations only holding locally). The authors should reflect on how these assumptions might be violated in practice and what the implications would be.
        \item The authors should reflect on the scope of the claims made, e.g., if the approach was only tested on a few datasets or with a few runs. In general, empirical results often depend on implicit assumptions, which should be articulated.
        \item The authors should reflect on the factors that influence the performance of the approach. For example, a facial recognition algorithm may perform poorly when image resolution is low or images are taken in low lighting. Or a speech-to-text system might not be used reliably to provide closed captions for online lectures because it fails to handle technical jargon.
        \item The authors should discuss the computational efficiency of the proposed algorithms and how they scale with dataset size.
        \item If applicable, the authors should discuss possible limitations of their approach to address problems of privacy and fairness.
        \item While the authors might fear that complete honesty about limitations might be used by reviewers as grounds for rejection, a worse outcome might be that reviewers discover limitations that aren't acknowledged in the paper. The authors should use their best judgment and recognize that individual actions in favor of transparency play an important role in developing norms that preserve the integrity of the community. Reviewers will be specifically instructed to not penalize honesty concerning limitations.
    \end{itemize}

\item {\bf Theory assumptions and proofs}
    \item[] Question: For each theoretical result, does the paper provide the full set of assumptions and a complete (and correct) proof?
    \item[] Answer: \answerYes{} 
    \item[] Justification: Every theorem and lemma in this paper has a complete and rigorous proof, either in the main text or deferred to appendix.
    \item[] Guidelines:
    \begin{itemize}
        \item The answer \answerNA{} means that the paper does not include theoretical results. 
        \item All the theorems, formulas, and proofs in the paper should be numbered and cross-referenced.
        \item All assumptions should be clearly stated or referenced in the statement of any theorems.
        \item The proofs can either appear in the main paper or the supplemental material, but if they appear in the supplemental material, the authors are encouraged to provide a short proof sketch to provide intuition. 
        \item Inversely, any informal proof provided in the core of the paper should be complemented by formal proofs provided in appendix or supplemental material.
        \item Theorems and Lemmas that the proof relies upon should be properly referenced. 
    \end{itemize}

    \item {\bf Experimental result reproducibility}
    \item[] Question: Does the paper fully disclose all the information needed to reproduce the main experimental results of the paper to the extent that it affects the main claims and/or conclusions of the paper (regardless of whether the code and data are provided or not)?
    \item[] Answer: \answerNA{} 
    \item[] Justification: This paper is purely theoretical by the deadline of submission.
    \item[] Guidelines:
    \begin{itemize}
        \item The answer \answerNA{} means that the paper does not include experiments.
        \item If the paper includes experiments, a \answerNo{} answer to this question will not be perceived well by the reviewers: Making the paper reproducible is important, regardless of whether the code and data are provided or not.
        \item If the contribution is a dataset and\slash or model, the authors should describe the steps taken to make their results reproducible or verifiable. 
        \item Depending on the contribution, reproducibility can be accomplished in various ways. For example, if the contribution is a novel architecture, describing the architecture fully might suffice, or if the contribution is a specific model and empirical evaluation, it may be necessary to either make it possible for others to replicate the model with the same dataset, or provide access to the model. In general. releasing code and data is often one good way to accomplish this, but reproducibility can also be provided via detailed instructions for how to replicate the results, access to a hosted model (e.g., in the case of a large language model), releasing of a model checkpoint, or other means that are appropriate to the research performed.
        \item While NeurIPS does not require releasing code, the conference does require all submissions to provide some reasonable avenue for reproducibility, which may depend on the nature of the contribution. For example
        \begin{enumerate}
            \item If the contribution is primarily a new algorithm, the paper should make it clear how to reproduce that algorithm.
            \item If the contribution is primarily a new model architecture, the paper should describe the architecture clearly and fully.
            \item If the contribution is a new model (e.g., a large language model), then there should either be a way to access this model for reproducing the results or a way to reproduce the model (e.g., with an open-source dataset or instructions for how to construct the dataset).
            \item We recognize that reproducibility may be tricky in some cases, in which case authors are welcome to describe the particular way they provide for reproducibility. In the case of closed-source models, it may be that access to the model is limited in some way (e.g., to registered users), but it should be possible for other researchers to have some path to reproducing or verifying the results.
        \end{enumerate}
    \end{itemize}

\item {\bf Open access to data and code}
    \item[] Question: Does the paper provide open access to the data and code, with sufficient instructions to faithfully reproduce the main experimental results, as described in supplemental material?
    \item[] Answer: \answerNA{} 
    \item[] Justification: This paper is purely theoretical by the deadline of submission.
    \item[] Guidelines:
    \begin{itemize}
        \item The answer \answerNA{} means that paper does not include experiments requiring code.
        \item Please see the NeurIPS code and data submission guidelines (\url{https://neurips.cc/public/guides/CodeSubmissionPolicy}) for more details.
        \item While we encourage the release of code and data, we understand that this might not be possible, so \answerNo{} is an acceptable answer. Papers cannot be rejected simply for not including code, unless this is central to the contribution (e.g., for a new open-source benchmark).
        \item The instructions should contain the exact command and environment needed to run to reproduce the results. See the NeurIPS code and data submission guidelines (\url{https://neurips.cc/public/guides/CodeSubmissionPolicy}) for more details.
        \item The authors should provide instructions on data access and preparation, including how to access the raw data, preprocessed data, intermediate data, and generated data, etc.
        \item The authors should provide scripts to reproduce all experimental results for the new proposed method and baselines. If only a subset of experiments are reproducible, they should state which ones are omitted from the script and why.
        \item At submission time, to preserve anonymity, the authors should release anonymized versions (if applicable).
        \item Providing as much information as possible in supplemental material (appended to the paper) is recommended, but including URLs to data and code is permitted.
    \end{itemize}

\item {\bf Experimental setting/details}
    \item[] Question: Does the paper specify all the training and test details (e.g., data splits, hyperparameters, how they were chosen, type of optimizer) necessary to understand the results?
    \item[] Answer: \answerNA{} 
    \item[] Justification: This paper is purely theoretical by the deadline of submission.
    \item[] Guidelines:
    \begin{itemize}
        \item The answer \answerNA{} means that the paper does not include experiments.
        \item The experimental setting should be presented in the core of the paper to a level of detail that is necessary to appreciate the results and make sense of them.
        \item The full details can be provided either with the code, in appendix, or as supplemental material.
    \end{itemize}

\item {\bf Experiment statistical significance}
    \item[] Question: Does the paper report error bars suitably and correctly defined or other appropriate information about the statistical significance of the experiments?
    \item[] Answer: \answerNA{} 
    \item[] Justification: This paper is purely theoretical by the deadline of submission.
    \item[] Guidelines:
    \begin{itemize}
        \item The answer \answerNA{} means that the paper does not include experiments.
        \item The authors should answer \answerYes{} if the results are accompanied by error bars, confidence intervals, or statistical significance tests, at least for the experiments that support the main claims of the paper.
        \item The factors of variability that the error bars are capturing should be clearly stated (for example, train/test split, initialization, random drawing of some parameter, or overall run with given experimental conditions).
        \item The method for calculating the error bars should be explained (closed form formula, call to a library function, bootstrap, etc.)
        \item The assumptions made should be given (e.g., Normally distributed errors).
        \item It should be clear whether the error bar is the standard deviation or the standard error of the mean.
        \item It is OK to report 1-sigma error bars, but one should state it. The authors should preferably report a 2-sigma error bar than state that they have a 96\% CI, if the hypothesis of Normality of errors is not verified.
        \item For asymmetric distributions, the authors should be careful not to show in tables or figures symmetric error bars that would yield results that are out of range (e.g., negative error rates).
        \item If error bars are reported in tables or plots, the authors should explain in the text how they were calculated and reference the corresponding figures or tables in the text.
    \end{itemize}

\item {\bf Experiments compute resources}
    \item[] Question: For each experiment, does the paper provide sufficient information on the computer resources (type of compute workers, memory, time of execution) needed to reproduce the experiments?
    \item[] Answer: \answerNA{} 
    \item[] Justification: This paper is purely theoretical by the deadline of submission.
    \item[] Guidelines:
    \begin{itemize}
        \item The answer \answerNA{} means that the paper does not include experiments.
        \item The paper should indicate the type of compute workers CPU or GPU, internal cluster, or cloud provider, including relevant memory and storage.
        \item The paper should provide the amount of compute required for each of the individual experimental runs as well as estimate the total compute. 
        \item The paper should disclose whether the full research project required more compute than the experiments reported in the paper (e.g., preliminary or failed experiments that didn't make it into the paper). 
    \end{itemize}
    
\item {\bf Code of ethics}
    \item[] Question: Does the research conducted in the paper conform, in every respect, with the NeurIPS Code of Ethics \url{https://neurips.cc/public/EthicsGuidelines}?
    \item[] Answer: \answerYes{} 
    \item[] Justification: This theoretical research is fully in line with the NeurIPS Code of Ethics.
    \item[] Guidelines:
    \begin{itemize}
        \item The answer \answerNA{} means that the authors have not reviewed the NeurIPS Code of Ethics.
        \item If the authors answer \answerNo, they should explain the special circumstances that require a deviation from the Code of Ethics.
        \item The authors should make sure to preserve anonymity (e.g., if there is a special consideration due to laws or regulations in their jurisdiction).
    \end{itemize}

\item {\bf Broader impacts}
    \item[] Question: Does the paper discuss both potential positive societal impacts and negative societal impacts of the work performed?
    \item[] Answer: \answerNo{} 
    \item[] Justification: The purpose of this research is to advance the theory of reinforcement learning. It may have downstream impacts, none of which is worth a specific discussion.
    \item[] Guidelines:
    \begin{itemize}
        \item The answer \answerNA{} means that there is no societal impact of the work performed.
        \item If the authors answer \answerNA{} or \answerNo, they should explain why their work has no societal impact or why the paper does not address societal impact.
        \item Examples of negative societal impacts include potential malicious or unintended uses (e.g., disinformation, generating fake profiles, surveillance), fairness considerations (e.g., deployment of technologies that could make decisions that unfairly impact specific groups), privacy considerations, and security considerations.
        \item The conference expects that many papers will be foundational research and not tied to particular applications, let alone deployments. However, if there is a direct path to any negative applications, the authors should point it out. For example, it is legitimate to point out that an improvement in the quality of generative models could be used to generate Deepfakes for disinformation. On the other hand, it is not needed to point out that a generic algorithm for optimizing neural networks could enable people to train models that generate Deepfakes faster.
        \item The authors should consider possible harms that could arise when the technology is being used as intended and functioning correctly, harms that could arise when the technology is being used as intended but gives incorrect results, and harms following from (intentional or unintentional) misuse of the technology.
        \item If there are negative societal impacts, the authors could also discuss possible mitigation strategies (e.g., gated release of models, providing defenses in addition to attacks, mechanisms for monitoring misuse, mechanisms to monitor how a system learns from feedback over time, improving the efficiency and accessibility of ML).
    \end{itemize}
    
\item {\bf Safeguards}
    \item[] Question: Does the paper describe safeguards that have been put in place for responsible release of data or models that have a high risk for misuse (e.g., pre-trained language models, image generators, or scraped datasets)?
    \item[] Answer: \answerNA{} 
    \item[] Justification: We believe that our theoretical analysis poses no such risk.
    \item[] Guidelines:
    \begin{itemize}
        \item The answer \answerNA{} means that the paper poses no such risks.
        \item Released models that have a high risk for misuse or dual-use should be released with necessary safeguards to allow for controlled use of the model, for example by requiring that users adhere to usage guidelines or restrictions to access the model or implementing safety filters. 
        \item Datasets that have been scraped from the Internet could pose safety risks. The authors should describe how they avoided releasing unsafe images.
        \item We recognize that providing effective safeguards is challenging, and many papers do not require this, but we encourage authors to take this into account and make a best faith effort.
    \end{itemize}

\item {\bf Licenses for existing assets}
    \item[] Question: Are the creators or original owners of assets (e.g., code, data, models), used in the paper, properly credited and are the license and terms of use explicitly mentioned and properly respected?
    \item[] Answer: \answerNA{} 
    \item[] Justification: We do not use any existing assets.
    \item[] Guidelines:
    \begin{itemize}
        \item The answer \answerNA{} means that the paper does not use existing assets.
        \item The authors should cite the original paper that produced the code package or dataset.
        \item The authors should state which version of the asset is used and, if possible, include a URL.
        \item The name of the license (e.g., CC-BY 4.0) should be included for each asset.
        \item For scraped data from a particular source (e.g., website), the copyright and terms of service of that source should be provided.
        \item If assets are released, the license, copyright information, and terms of use in the package should be provided. For popular datasets, \url{paperswithcode.com/datasets} has curated licenses for some datasets. Their licensing guide can help determine the license of a dataset.
        \item For existing datasets that are re-packaged, both the original license and the license of the derived asset (if it has changed) should be provided.
        \item If this information is not available online, the authors are encouraged to reach out to the asset's creators.
    \end{itemize}

\item {\bf New assets}
    \item[] Question: Are new assets introduced in the paper well documented and is the documentation provided alongside the assets?
    \item[] Answer: \answerNA{} 
    \item[] Justification: This paper does not release new assets.
    \item[] Guidelines:
    \begin{itemize}
        \item The answer \answerNA{} means that the paper does not release new assets.
        \item Researchers should communicate the details of the dataset\slash code\slash model as part of their submissions via structured templates. This includes details about training, license, limitations, etc. 
        \item The paper should discuss whether and how consent was obtained from people whose asset is used.
        \item At submission time, remember to anonymize your assets (if applicable). You can either create an anonymized URL or include an anonymized zip file.
    \end{itemize}

\item {\bf Crowdsourcing and research with human subjects}
    \item[] Question: For crowdsourcing experiments and research with human subjects, does the paper include the full text of instructions given to participants and screenshots, if applicable, as well as details about compensation (if any)? 
    \item[] Answer: \answerNA{} 
    \item[] Justification: This paper does not involve crowdsourcing nor research with human subjects.
    \item[] Guidelines:
    \begin{itemize}
        \item The answer \answerNA{} means that the paper does not involve crowdsourcing nor research with human subjects.
        \item Including this information in the supplemental material is fine, but if the main contribution of the paper involves human subjects, then as much detail as possible should be included in the main paper. 
        \item According to the NeurIPS Code of Ethics, workers involved in data collection, curation, or other labor should be paid at least the minimum wage in the country of the data collector. 
    \end{itemize}

\item {\bf Institutional review board (IRB) approvals or equivalent for research with human subjects}
    \item[] Question: Does the paper describe potential risks incurred by study participants, whether such risks were disclosed to the subjects, and whether Institutional Review Board (IRB) approvals (or an equivalent approval/review based on the requirements of your country or institution) were obtained?
    \item[] Answer: \answerNA{} 
    \item[] Justification: This paper does not involve crowdsourcing nor research with human subjects.
    \item[] Guidelines:
    \begin{itemize}
        \item The answer \answerNA{} means that the paper does not involve crowdsourcing nor research with human subjects.
        \item Depending on the country in which research is conducted, IRB approval (or equivalent) may be required for any human subjects research. If you obtained IRB approval, you should clearly state this in the paper. 
        \item We recognize that the procedures for this may vary significantly between institutions and locations, and we expect authors to adhere to the NeurIPS Code of Ethics and the guidelines for their institution. 
        \item For initial submissions, do not include any information that would break anonymity (if applicable), such as the institution conducting the review.
    \end{itemize}

\item {\bf Declaration of LLM usage}
    \item[] Question: Does the paper describe the usage of LLMs if it is an important, original, or non-standard component of the core methods in this research? Note that if the LLM is used only for writing, editing, or formatting purposes and does \emph{not} impact the core methodology, scientific rigor, or originality of the research, declaration is not required.
    \item[] Answer: \answerNA{} 
    \item[] Justification: We only use LLM for writing purpose and the core method development in this research does not involve LLMs as any important, original, or non-standard components.
    \item[] Guidelines:
    \begin{itemize}
        \item The answer \answerNA{} means that the core method development in this research does not involve LLMs as any important, original, or non-standard components.
        \item Please refer to our LLM policy in the NeurIPS handbook for what should or should not be described.
    \end{itemize}

\end{enumerate}

\end{document}